\documentclass[10pt,journal,compsoc]{IEEEtran}

\usepackage{ifpdf}
\ifCLASSOPTIONcompsoc
  \usepackage[nocompress]{cite}
\else
  \usepackage{cite}
\fi

\ifCLASSINFOpdf
\else
\fi

\usepackage{amsmath}
\usepackage{amsthm}
\usepackage{amssymb} 
\usepackage{array}
\usepackage{algorithm}
\usepackage{algpseudocode}
\usepackage{subfigure}
\usepackage{graphicx}
\usepackage{xcolor}
\usepackage{booktabs}
\usepackage{threeparttable}
\usepackage{multirow}
\usepackage{arydshln}
\usepackage{microtype}
\usepackage{enumitem}
\usepackage[hidelinks]{hyperref}

\theoremstyle{plain}
\newtheorem{prop}{Proposition}

\newcommand{\bfb}{\mathbf{b}}
\newcommand{\bfc}{\mathbf{c}}
\newcommand{\bff}{\mathbf{f}}
\newcommand{\bfG}{\mathbf{G}}
\newcommand{\bfp}{\mathbf{p}}
\newcommand{\bfw}{\mathbf{w}}
\newcommand{\bfW}{\mathbf{W}}
\newcommand{\bfx}{\mathbf{x}}

\newcommand{\minx}{x_{\textrm{min}}}
\newcommand{\maxx}{x_{\textrm{max}}}
\newcommand{\miny}{y_{\textrm{min}}}
\newcommand{\maxy}{y_{\textrm{max}}}
\newcommand{\minz}{z_{\textrm{min}}}
\newcommand{\maxz}{z_{\textrm{max}}}

\newcommand{\paraheading}[1]{\textbf{{#1}.~}}
\newcommand{\Rnum}[1]{\uppercase\expandafter{\romannumeral #1\relax}}
\newcommand{\rnum}[1]{\lowercase\expandafter{\romannumeral #1\relax}}
\newcommand{\trainsetone}{\textsc{Training-Set-\Rnum{1}}}
\newcommand{\trainsettwo}{\textsc{Training-Set-\Rnum{2}}}
\newcommand{\modela}{\textsc{Model-A}}
\newcommand{\modelb}{\textsc{Model-B}}
\newcommand{\modelc}{\textsc{Model-C}}
\newcommand{\modeld}{\textsc{Model-D}}

\definecolor{hlcolor}{RGB}{210,52,73}
\begin{document}

\title{Robust RGB-D Face Recognition Using Attribute-Aware Loss}

\author{Luo~Jiang,~
        Juyong~Zhang$^\dagger$,~\IEEEmembership{Member,~IEEE,}
        and~Bailin~Deng,~\IEEEmembership{Member,~IEEE}% <-this % stops a space
\IEEEcompsocitemizethanks{\IEEEcompsocthanksitem L. Jiang and J. Zhang are with School of Mathematical Sciences, University of Science and Technology of China. The research work was done when Luo Jiang did his internship at Beijing Dilusense Technology Corporation. %\protect\\
\IEEEcompsocthanksitem B. Deng is with School of Computer Science and Informatics, Cardiff University.}% <-this % stops an unwanted space
\thanks{$^\dagger$Corresponding author. Email: {\texttt{juyong@ustc.edu.cn}}.}
}

\markboth{~}
{Jiang \MakeLowercase{\textit{et al.}}: Robust RGB-D Face Recognition Using Attribute-Aware Loss}

\IEEEtitleabstractindextext{%
\begin{abstract}
Existing convolutional neural network (CNN) based face recognition algorithms typically learn a discriminative feature mapping, using a loss function that enforces separation of features from different classes and/or aggregation of features within the same class. However, they may suffer from bias in the training data such as uneven sampling density, because they optimize the adjacency relationship of the learned features without considering the proximity of the underlying faces. Moreover, since they only use facial images for training, the learned feature mapping may not correctly indicate the relationship of other attributes such as gender and ethnicity, which can be important for some face recognition applications. In this paper, we propose a new CNN-based face recognition approach that incorporates such attributes into the training process. Using an attribute-aware loss function that regularizes the feature mapping using attribute proximity, our approach learns more discriminative features that are correlated with the attributes. We train our face recognition model on a large-scale RGB-D data set with over 100K identities captured under real application conditions. By comparing our approach with other methods on a variety of experiments, we demonstrate that depth channel and attribute-aware loss greatly improve the accuracy and robustness of face recognition.
\end{abstract}

}

\maketitle
\IEEEdisplaynontitleabstractindextext
\IEEEpeerreviewmaketitle

\IEEEraisesectionheading{\section{Introduction}\label{sec:introduction}}
\IEEEPARstart{C}{onvolutional} neural networks (CNNs) play a significant role in face analysis tasks, such as landmarks detection~\cite{sun2013deep,zhang2016joint}, face recognition~\cite{huang2014labeled, wolf2011face, kemelmacher2016megaface} and 3D face reconstruction~\cite{richardson2016learning,Guo20183DFace}. With the emergence of large public face data sets~\cite{yi2014learning} and sophisticated networks~\cite{szegedy2015going, he2016deep}, the problem of face recognition has gained lots of attention and developed rapidly. At present, some mainstream methods already outperform humans on certain benchmark datasets such as~\cite{huang2014labeled}. These methods usually map faces to discriminative feature vectors in a high-dimensional Euclidean space, to determine whether a pair of faces belong to the same category. For example, deep metric learning methods (such as contrastive loss~\cite{hadsell2006dimensionality} or triplet loss~\cite{schroff2015facenet}) usually train a CNN by comparing pairs or triplets of facial images to learn discriminative features. Later, different variants of the softmax loss~\cite{taigman2014deepface, wen2016discriminative, ranjan2017l2, wang2017normface, liu2017sphereface} are used as supervision signals in CNNs to extract discriminative features, which achieve excellent performance under the protocol of small training set. These methods~\cite{schroff2015facenet,wen2016discriminative,liu2017sphereface} utilize CNNs to learn strong discriminative deep features, using loss functions that enforce either intra-class compactness or inter-class dispersion.

Although the above two categories of methods have achieved remarkable performance, they still have their own limitations. First, the contrastive loss and triplet loss suffer from slow convergence due to the construction of a large number of pairs or triplets. To accelerate convergence,~\cite{sohn2016improved} proposed a $(N+1)$-tuple loss that increases the number of negative examples. However, this loss still requires complex recombination of training samples. In comparison, the softmax loss and its variants have no such requirement on the training data and converge more quickly. The center loss~\cite{wen2016discriminative} is the first to add soft constraints on deep features in the softmax loss to minimize the intra-class variations, significantly improving the performance of softmax loss. Afterward, the angular softmax loss~\cite{liu2017sphereface} imposed discriminative constraints on a hypersphere manifold, which further improved the performance of softmax loss. However, by enforcing intra-class aggregation and inter-class separation among the training data, existing variants of softmax loss encourage the uniform distribution of feature vectors for the training data, even though the training data may not be sampled uniformly. As a result, the proximity between the learned feature vectors for two test data may not correctly indicate the proximity between their underlying faces, which can affect the accuracy of face recognition algorithms based on feature proximity. To address this issue, we propose an attribute-aware loss function that regularizes the learned feature mapping using other attributes such as gender, ethnicity, and age. The proposed loss function imposes a global linear relation between the feature difference and the attribute difference between nearby training data, such that feature vectors for facial data with similar attributes are driven towards each other. In addition, as these attributes are correlated with facial geometry and appearance, the attribute-aware loss also implicitly regularizes the feature proximity with respect to the facial proximity, which helps to account for potential sampling bias in the training set.  

In addition, although existing RGB image-based face recognition methods have achieved great success, they rely solely on the appearance information and may suffer from poor lighting conditions such as dark environments. On the other hand, the depth image captured by RGB-D sensors such as PrimeSense sensors provides additional geometric information that is independent of illumination, which can help to improve the robustness of recognition. To this end, we develop a CNN-based RGB-D face recognition approach, by first aligning the depth map with the RGB image grid and normalizing the depth values to the same range as the RGB values, and then feeding the resulting RGB-D values into CNNs for training and testing. Unlike existing RGB-D based deep learning approaches~\cite{LeeCTL16, HCSC18} that only use small training data sets with less than 1K identities, we train our model on a large RGB-D data set with over 100K identities, where the resulting model achieves more robust performance than RGB based approaches.  

Combining the RGB-D approach with the attribute-aware loss function, our new method greatly improves the robustness and accuracy of facial recognition. We tested our method on several datasets, with different identities in diverse facial expressions and lighting conditions. Our method performs consistently better than state-of-the-art approaches that only rely on RGB information and do not consider additional attributes.

To summarize, this paper makes the following major contributions:
\begin{itemize}
\item We propose an attribute-aware loss function for CNN-based face recognition, which regularizes the distribution of learned feature vectors with respect to additional attributes and improves the accuracy of recognition results. To the best of our knowledge, this is the first method that utilizes non-facial attributes to improve CNN-based face recognition feature training. 
\item For neural network training and testing, we construct a large-scale RGB-D face dataset including more than 100k identities mainly with the frontal pose, and a relatively small RGB-D dataset with 952 identities with various poses. This is the first result that verifies the effectiveness of CNN-based RGB-D face recognition with large training data sets.
\end{itemize}

\section{Related Work}

Face recognition is a classical research topic in pattern recognition and computer vision, with applications in many areas like biometrics, surveillance system, and information security. For a comprehensive review of 2D face recognition and 3D face recognition methods,
one may refer to~\cite{ParkhiVZ15, ABATE20071885}. This section briefly reviews those techniques that are closely related to our work.

\subsection{Deep Learning based Face Recognition}

In the past few years, deep learning based face recognition is one of the most active research areas. In this part, we mainly discuss the loss functions used in these methods.

\textbf{Metric Learning.}
Metric learning~\cite{xing2002distance, weinberger2009distance, Wang2011kernel} attempts to optimize a parametric notion of distance in a fully/weakly/semi-supervised way such that the similar objects are nearby and dissimilar objects are far apart on a target space. In~\cite{xing2002distance}, the learning is done by finding a Mahalanobis distance with a matrix parameter when given some similar pairs of samples. In order to handle more challenging problems, kernel tricks~\cite{Wang2011kernel,jain2012metric} had been introduced in metric learning to extract nonlinear embeddings. In recent years, more discriminative features can be learned with advanced network architectures that minimize some loss functions based on Euclidean distance, such as contrastive loss~\cite{hadsell2006dimensionality} and triplet loss~\cite{schroff2015facenet}. Moreover, these loss functions can be improved by allowing joint comparison among more than one negative example~\cite{sohn2016improved} or minimizing the overall classification error~\cite{kumar2016learning}.

\textbf{Classification Losses.}
The most commonly used classification loss is the softmax loss that maps images to deep features and then to predicted labels. Krizhevsky et al.~\cite{krizhevsky2012imagenet} first observed that CNNs trained with softmax loss can produce discriminative feature vectors, which has also been confirmed by other works~\cite{sharif2014cnn}. However, softmax loss mainly encourages inter-class dispersion, and thus cannot induce strong discriminative features. To enhance the discrimination power of deep features, Wen et al.~\cite{wen2016discriminative} proposed center loss to enforce intra-class aggregation as well as inter-class dispersion. Meanwhile, Ranjan et al.~\cite{ranjan2017l2} observed that the softmax loss is biased to the sample distribution, i.e., fitting well to high-quality faces but ignoring the low-quality faces. Adding $\ell_2$-constraints on features to the softmax loss can make the resulting features as discriminative as those trained with center loss. Afterward, Liu and colleagues~\cite{liu2016large, liu2017sphereface} further improved the features by incorporating an angular margin instead of the Euclidean margin into softmax loss to enhance the inter-class margin and compressing the intra-class angular distribution simultaneously.

\subsection{Face Recognition with Attributes}

Besides the feature vectors extracted from CNN, other attributes can also be utilized in face recognition tasks. An early study~\cite{KumarBBN09} trained 65 ``attribute'' SVM classifiers to recognize the traits of input facial images such as gender, age, race, and hair color, which are then fused with other features for face recognition. In the context of deep learning, attribute-enhanced face recognition does not gain too much attention. One related work~\cite{SamangoueiC16} is to exploit CNN based attribute features for authentication on mobile devices, and the facial attributes are trained by a multi-task, partly based Deep Convolutional Neural Network architecture. Hu et.al~\cite{hu2017attribute} systematically study the problem of how to fuse face recognition features and facial attribute features to enhance face recognition performance. They reformulate feature fusion as a gated two-stream neural network, which can be efficiently optimized by neural network learning.

Based on the assumption that attributes like gender, age and pose could share low-level features from the representation learning perspective, some studies investigate multi-task learning~\cite{rudd2016moon, RanjanSCC17} and show that such attributes could help the face recognition task. In our method, different from the above attribute fusion and multi-task learning methods, the attributes are directly used to guide the face recognition feature learning in the training stage, and they are not needed during the testing stage.

\subsection{RGB-D Face Recognition}
In recent years, RGB-D based face recognition has attracted increasing attention because of its robustness in an unconstrained environment. Hsu et al.~\cite{HsuLPW14} considered a scenario in which the gallery is a pair of RGB-D images while the probe is a single RGB image captured by a regular camera without the depth channel. They proposed an approach that reconstructs a 3D face from an RGB-D image for each subject in the gallery, aligns the reconstructed 3D model to a probe using facial landmarks, and recognizes the probe using sparse representation based classification. Zhang et al.~\cite{HCSC18} further considered the problem of multi-modality matching (e.g., RGB-D probe vs. RGB-D gallery) and cross-modality matching (e.g., RGB probe vs. RGB-D) in the same framework. They proposed an approach for RGB-D face recognition that is able to learn complementary features from multiple modalities and common features between different modalities. For the RGB-D vs. RGB-D problem, Goswami et al.~\cite{GoswamiVS14} proposed to compute an RGB-D image descriptor based on entropy and based on the entropy and saliency, as well as geometric facial attributes from the depth map; then the descriptor and the attributes are fused to perform recognition. Li et al.~\cite{LiXMLK16} proposed a multi-channel weighted sparse coding method on the hand-crafted features for RGB-D face recognition. 

Although it is straightforward to extend deep learning based face recognition methods from RGB images to RGB-D images, currently there are no large-scale public RGB-D data sets that can be used for training, which limits the practical applications of these approaches. For example, the model proposed in~\cite{HCSC18} is trained on a dataset with less than 1K identities. To handle this problem, Lee et al.~\cite{LeeCTL16} proposed to first train the deep network with a color face dataset, and then fine-tune it on depth face images for transfer learning.  

\section{Method}
\subsection{Revisiting the Variants of Softmax Loss}
Given a training data set $\{\bfx_i\}_{i=1}^{N}$ with $\bfx_i \in \mathbb{R}^{m\times n}$, and their corresponding labels $\{y_i\}_{i=1}^{N}$ with $y_i \in \mathcal{I}\triangleq \{1,...,C\}$, the following classical softmax loss function is widely used in face recognition tasks
\begin{equation}
	L_s = -\sum_{i=1}^{N} \log (\frac{\exp(\bfw_{y_i}^{T}f(\bfx_i) + b_{y_i})}{\sum_{j=1}^{C} \exp(\bfw_{j}^{T}f(\bfx_i) + b_j)}),
	\label{eq:softmax_loss}
\end{equation}
where $f(\cdot):\mathbb{R}^{m\times n} \rightarrow \mathbb{R}^{K \times 1}$ is the learned feature mapping by training CNNs, $K$ is the dimension of deep feature $f(\bfx_i)$. $\bfW=[\bfw_1,...,\bfw_C]\in \mathbb{R}^{K \times C}$ and $\bfb=[b_1, ..., b_C] \in \mathbb{R}^{1 \times C}$ are the weights and biases in the last fully connected layer, and $\{\bfx_i\}$ can be color or depth images of faces. We denote $\bff_i = f(\bfx_i)$ for simplicity. Typically, during the test phase the mapping $f(\cdot)$ is applied on an image pair ${\left(\bfx_i, \bfx_j\right)}$ to extract two deep features $(\bff_i, \bff_j)$, and the Euclidean distance or cosine distance between the features are computed to determine the similarity between the image pair. Separable features can be learned using softmax loss, but they are not discriminative enough for face recognition.

To learn more discriminative features, several variants of softmax loss have been developed by enlarging the inter-class margin and reducing the intra-class variation. Among them, the center loss~\cite{wen2016discriminative} requires the deep features of each class to gather towards their respective centers $\{\bfc_i\}_{i=1}^{N}$:
\begin{equation}
L_c = \frac{1}{2}\sum_i \|\bff_i-\bfc_{y_i}\|_2^2.
\end{equation}
With the angular softmax loss~\cite{liu2017sphereface}, deep features of each class are compressed using the angular margin $\tau$ instead of the Euclidean margin:
\begin{equation}
L_{\tau} = -\sum_{i=1}^{N} \log \frac{\exp({\|\bff_i\|\cos(\tau\theta_{y_i,i})})}{\exp({\|\bff_i\|\cos(\tau\theta_{y_i,i})}) + \sum\limits_{j \neq y_i} \exp({\|\bff_i\| \cos(\theta_{j,i})})},
\label{eq:AngularSoftMax}
\end{equation}
where $\theta_{j,i}$ is the angle between vectors $\mathbf{w}_j$ and $\mathbf{f}_i$.
There are other variants of softmax loss~\cite{wang2018additive, deng2018arcface} with a similar form as~\eqref{eq:AngularSoftMax}, where the margin and the angle are added instead of being multiplied.

\begin{figure*}[t!] 
	\centering
	\includegraphics[width=0.98\textwidth]{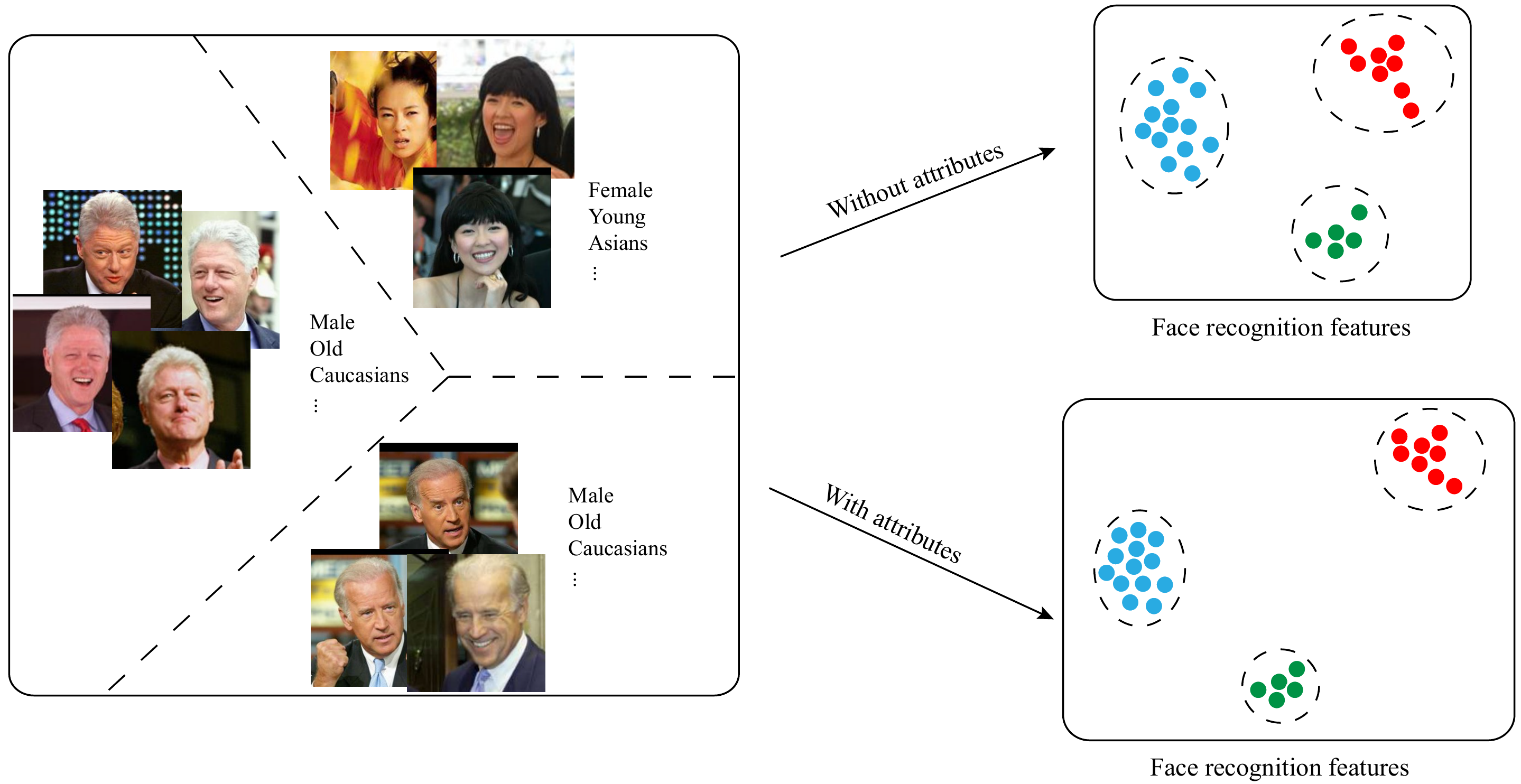}
	\caption{The face recognition features trained with only the softmax loss tend to be evenly distributed for the training data, which can become less effective when the training data are not evenly sampled from the face space. By augmenting the loss function with the attribute-aware loss term, the learned features have similar distribution as the attributes, which are better suited for recognition.}
	\label{fig:method}
\end{figure*}

\subsection{The Attribute-Aware Loss}
To achieve high accuracy for face recognition, it is desirable that the proximity between feature clusters of different classes is consistent with the proximity between the classes (i.e., the underlying faces). Ideally, the more dissimilar two faces are, the further apart their corresponding feature clusters should be. However, this is not guaranteed by the above variants of softmax loss according to our experimental observations (see Fig.~\ref{fig:toy_exp} and Fig.~\ref{fig:gender}). Since they minimize the intra-class variations and maximize the inter-class margins on the training data, the learned feature mappings tend to produce evenly distributed feature vectors for the training faces. On the other hand,
there is no guarantee that the facial images in the training set are evenly distributed in the full face space. As a result, when there exist large variations of sampling density in the training data set, the learned feature mapping may not correctly indicate the proximity of the underlying faces. To address this issue, we can try to introduce a loss function term that regularizes feature proximity with respect to face proximity. However, this is a challenging task as a facial image only reveals the underlying face shape from a certain view direction and can be affected by various factors such as lighting condition and sensor noises. As a result, it is difficult to reliably compute the proximity between two faces by only comparing their scanned images.

Besides the proximity of face shapes, it is also desirable that the learned feature mappings are related to the proximity between other attributes such as gender, ethnicity, and age. For example, if we compare a probe image against a database of facial images to identify $Q$ most likely matches via feature proximity, then it is preferable that all returned images are from persons with the same or similar attributes. The above variants of softmax loss cannot ensure this property either, as they only consider the facial images during the training process.

Motivated by these observations, we propose an attribute-aware loss term that regularizes feature proximity with respect to attribute proximity. Besides the label information, other attributes of the facial images like gender, ethnicity, and age are also given in the training data set. These attributes can be collected during training data construction, and they are independent of the imaging process. We represent the augmented attributes for a facial image $\mathbf{x}_i$ using a vector $\bfp_i \in \mathbb{R}^{H \times 1}$. Then our attribute-aware loss is formulated as
\begin{equation}
	L_a = \frac{1}{2} \sum_{\substack{y_i < y_j \\ d(\bfp_i, \bfp_j) < \tau}} \|(\bff_i-\bff_j) - \bfG (\bfp_i - \bfp_j)\|_2^2,
	\label{eq:AA_loss}
\end{equation}
where $\bfG \in \mathbb{R}^{K \times H}$ is a parameter matrix to be trained, $d(\bfp_i, \bfp_j)=\|\bfp_i-\bfp_j\|_2$ is the Euclidean distance between the two attribute vectors, and $\tau$ is a user-specified threshold. 

Intuitively, this loss term can drive feature clusters with similar attributes towards each other, via a global linear mapping $\bfG$ that relates the feature difference to the attribute difference. To the best of our knowledge, this is the first work in face recognition that optimizes adjacency of learned features using attribute proximity. As shown in Fig.~\ref{fig:method}, the learned feature clusters with similar attributes become closer after our adjacency optimization.
From another perspective, regularization using additional attributes can help the network pick up other useful cues for face recognition, because attributes such as gender, ethnicity, and age are highly correlated with facial shape and appearance. For example, there can be a notable difference between the facial appearance of two persons with different genders. Therefore, the attribute-aware loss can improve the learned feature mapping by implicitly utilizing the appearance variation related to these attributes.

\begin{figure*}[t!]	
	\subfigure[Softmax loss]{
	\begin{minipage}[t]{0.5\linewidth}
	\centering
	\includegraphics[width=0.95\columnwidth]{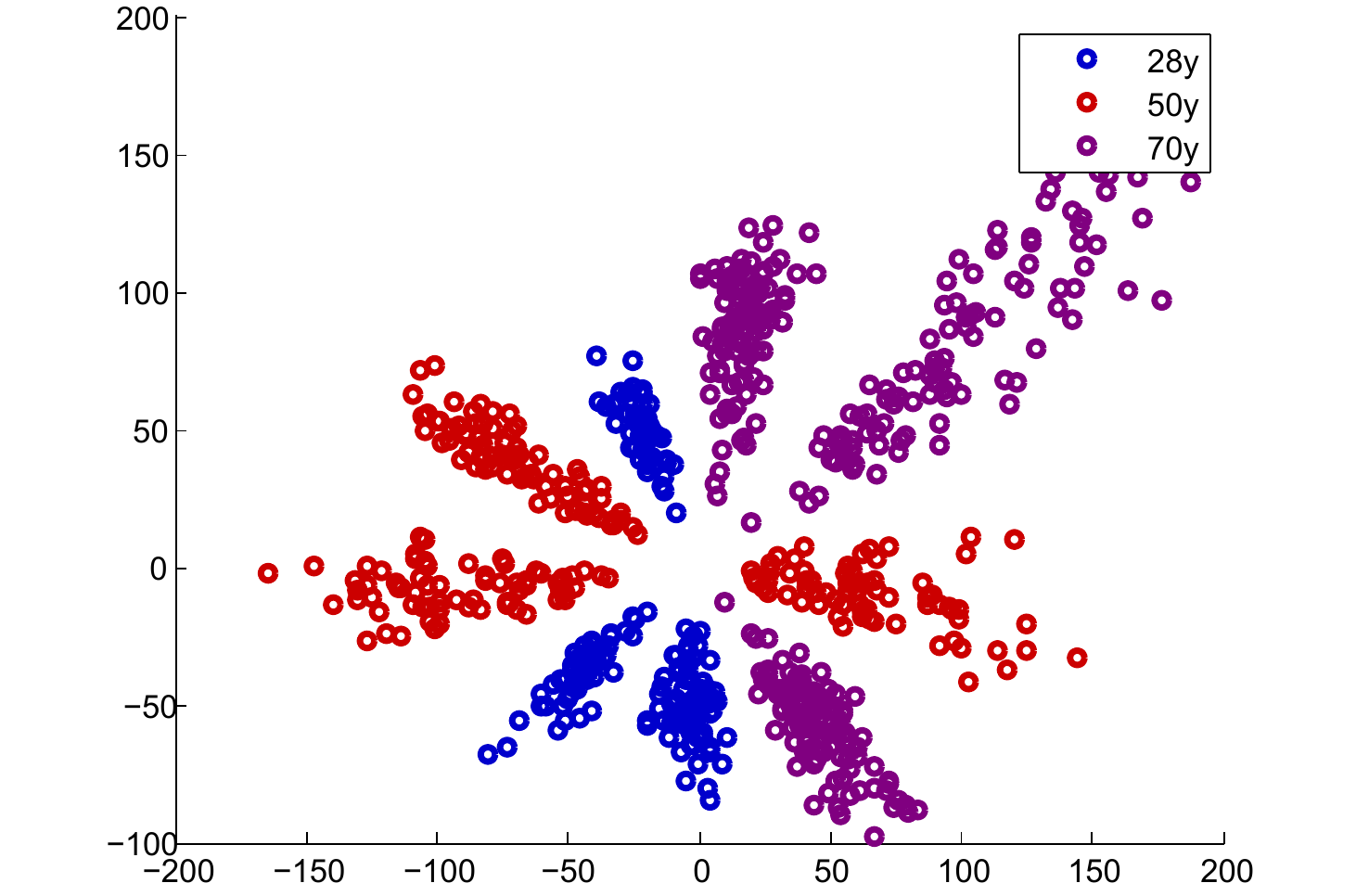}
	\label{fig:side:a}
	\end{minipage}
	}
	\subfigure[Softmax loss + Attribute-aware loss]{
	\begin{minipage}[t]{0.5\linewidth}
	\centering
	\includegraphics[width=0.95\columnwidth]{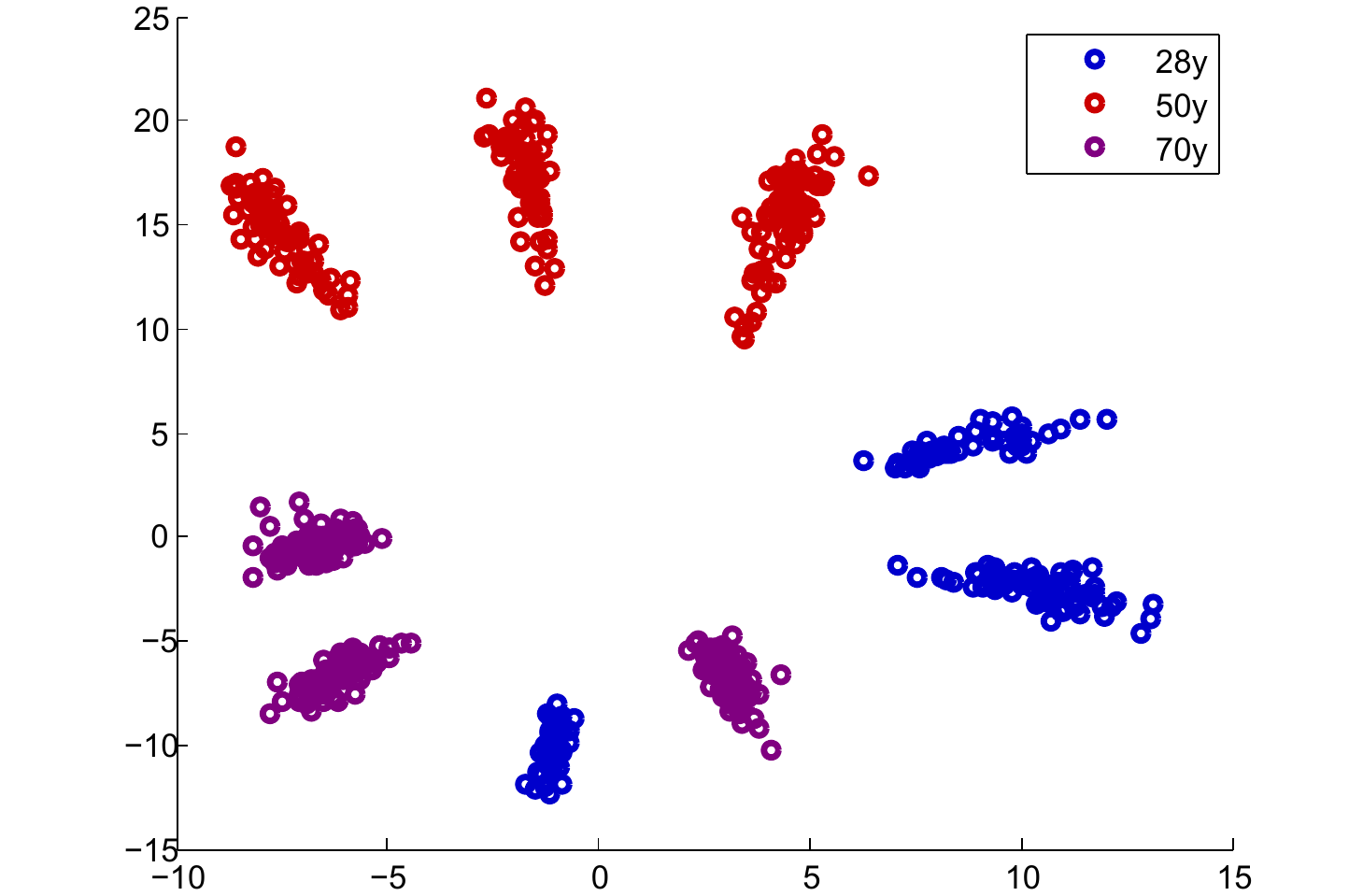}
	\label{fig:side:b}
	\end{minipage}
	}
	\caption{The distribution of deeply learned features under (a) softmax loss and (b) the joint supervision of softmax loss and attribute-aware loss. There are nine identities with three different ages. The points with different colors denote features with different ages.}
	\label{fig:toy_exp}
\end{figure*}

To better understand the attribute-aware loss, the following proposition clarifies how the difference between $\bfp_i$ and $\bfp_j$ influences the difference between $\bff_i$ and $\bff_j$ in the feature space.
\begin{prop}
\label{prop}
Let $\{\mathbf{f}_i\}_{i=0}^{N-1}$ be facial features extracted from training samples $\{\mathbf{x}_i\}_{i=0}^{N-1}$ with corresponding labels $\{y_i\}_{i=0}^{N-1}$ and attribute vectors $\{\mathbf{p}_i\}_{i=0}^{N-1}$. Assume that: 
\begin{enumerate}[label=(\roman*)]
	\item The number of training epochs is large enough to allow all attribute pairs $(\mathbf{p}_i, \mathbf{p}_j)$ that satisfy constrains $y_i \neq y_j$ and $d(\mathbf{p}_i, \mathbf{p}_j) = \|\mathbf{p}_i - \mathbf{p}_j\|_2 < \tau$ to appear sufficiently many times in training phase;
	\item For every identity, there exists at least another identity so that they can meet the above constraints;
	\item The training is convergent, the parameter matrix $\mathbf{G}$ is nonsingular, and $\|\mathbf{G}\|_2 \leq M$ where $M$ is a constant.
	\item The value
	\begin{equation}
	\epsilon \triangleq \max\limits_{\substack{y_i < y_j\\d(\mathbf{p}_i, \mathbf{p}_j)<\tau, \\ i,j \in \{0,\cdots,N-1\}}} \|(\mathbf{f}_i-\mathbf{f}_j) - \mathbf{G} (\mathbf{p}_i - \mathbf{p}_j)\|_2
	\label{eq:MaxEpsilon}
	\end{equation}
	is bounded.
\end{enumerate}
Let $\mathcal{S}_{k}$ denote the set of all facial features of the $k$-th identity. Then 
\begin{enumerate}[label=(\roman*)]
	\item For each $k$, the Euclidean distance between any pair in $\mathcal{S}_{k}$ has an upper bound which is linear with threshold $\tau$;
	\item If the Euclidean distance between the attribute vectors of the $k_1$-th identity and the $k_2$-th identity is smaller than $\tau$, then the Euclidean distance between their average features $\overline{\mathbf{f}}^{k_1} = (\sum_{\mathbf{f}_i \in \mathcal{S}_{k_1}} \mathbf{f}_i)/|\mathcal{S}_{k_1}|$ and 
	$\overline{\mathbf{f}}^{k_2} = (\sum_{\mathbf{f}_j \in \mathcal{S}_{k_2}} \mathbf{f}_j) / |\mathcal{S}_{k_2}|$ also has an upper bound which is linear with threshold $\tau$.
\end{enumerate}
\end{prop}
\begin{proof}
From Eq.~\eqref{eq:MaxEpsilon}, we have
\[
\|(\mathbf{f}_i - \mathbf{f}_j) - \mathbf{G} (\mathbf{p}_i - \mathbf{p}_j)\|_2 \leq \epsilon, ~~ \forall~y_i < y_j,~d(\mathbf{p}_i, \mathbf{p}_j) < \tau,
\]
Thus
\begin{align}
d(\mathbf{f}_i, \mathbf{f}_j) & \leq \|\mathbf{G} (\mathbf{p}_i - \mathbf{p}_j)\|_2 + \epsilon \leq \|\mathbf{G}\|_2 \|\mathbf{p}_i - \mathbf{p}_j\|_2 + \epsilon \nonumber \\
& < M\cdot\tau + \epsilon. \label{eq:FeatureDistBound}
\end{align}
Let $\mathbf{f}_{\alpha}$ and $\mathbf{f}_{\beta}$ be an arbitrary pair of features in $\mathcal{S}_{k}$. Then from assumption (\rnum{2}), we can find another feature $\mathbf{f}_{\gamma} \in \mathcal{S}_{k'}$ where $k' \neq k$, such that their corresponding attributes $\mathbf{p}_{\alpha}, \mathbf{p}_{\beta}, \mathbf{p}_{\gamma}$ satisfy $d(\mathbf{p}_{\alpha}, \mathbf{p}_{\gamma}) < \tau$ and  $d(\mathbf{p}_{\beta}, \mathbf{p}_{\gamma}) < \tau$.
Applying Eq.~\ref{eq:FeatureDistBound} to the pairs $(\mathbf{p}_{\alpha}, \mathbf{p}_{\gamma})$ and $(\mathbf{p}_{\beta}, \mathbf{p}_{\gamma})$, we obtain
\begin{equation}
	d(\mathbf{f}_{\alpha}, \mathbf{f}_{\beta}) \leq d(\mathbf{f}_{\alpha}, \mathbf{f}_{\gamma}) + d(\mathbf{f}_{\beta}, \mathbf{f}_{\gamma})
	< 2 (M \cdot \tau + \epsilon).
	\label{eq:IntraClassDistBound}
\end{equation}
According to assumptions (\rnum{3}) and (\rnum{4}), $M$ is a constant and $\epsilon$ is bounded. Therefore, Eq.~\eqref{eq:IntraClassDistBound} provides an upper bound for the Euclidean distance between any pair in $\mathcal{S}_{k}$ which is linear with $\tau$.

For the average features $\overline{\mathbf{f}}^{k_1}$ and $\overline{\mathbf{f}}^{k_2}$ of $\mathcal{S}_{k_1}$ and $\mathcal{S}_{k_2}$, their Euclidean distance satisfies
\begin{equation}
	d(\overline{\mathbf{f}}^{k_1}, \overline{\mathbf{f}}^{k_2}) \leq \max_{\substack{\mathbf{f}_i \in \mathcal{S}_{k_1} \\ \mathbf{f}_j \in \mathcal{S}_{k_2}}} d(\mathbf{f}_i, \mathbf{f}_j).
	\label{eq:AvergeFeatureDist}
\end{equation}
If the Euclidean distance between the attribute vectors of the $k_1$-th identity and the $k_2$-th identity is smaller than $\tau$, then by definition $d(\mathbf{p}_i, \mathbf{p}_j) < \tau$ for any $\mathbf{f}_i \in \mathcal{S}_{k_1}$ and $\mathbf{f}_j \in \mathcal{S}_{k_2}$. This implies $d(\mathbf{f}_i, \mathbf{f}_j) < M \cdot \tau + \epsilon$ according to Eq.~\eqref{eq:FeatureDistBound}. Applying this relation to Eq.~\eqref{eq:AvergeFeatureDist}, we obtain
\[
d(\overline{\mathbf{f}}^{k_1}, \overline{\mathbf{f}}^{k_2}) < M \cdot \tau + \epsilon,
\]
which provides an upper bound that is linear with $\tau$.
\end{proof}

Proposition.~\ref{prop} implies two properties. First, the attribute-aware loss layer can make intra-class features more compact than using the softmax loss only, similar to the center loss~\cite{wen2016discriminative}; the smaller the threshold $\tau$ is, the more compact the intra-class features may become. Sec.~\ref{subsec: param} evaluates the effects from different values of $\tau$. Second, for two identities with similar attributes, their corresponding feature clusters will not be far away from each other. This can be demonstrated by our experiment in Sec.~\ref{subsec:private}.

To showcase the effectiveness of the attribute-aware loss, we present a toy example on a very small RGB face dataset. This dataset is selected from our large-scale RGB-D face dataset presented in Section~\ref{subsec: details}, and contains only nine identities with the same gender and ethnicity but different ages. Three of the identities are aged 28, another three aged 50, and the remaining aged 70. We use ResNet-10~\cite{he2016deep} to train two models, one with softmax loss only, the other with both the softmax loss and the attribute-aware loss. Details of training with the combined softmax and attribute-aware loss are presented in Sec.~\ref{subsec:AttributeAwareLossTraining}. We reduce the output dimension of the penultimate fully connected layer to two, allowing us to directly plot the learned features in Fig.~\ref{fig:side:a} and Fig.~\ref{fig:side:b}. We can see that the coordinates of features in Fig.~\ref{fig:side:a} span a much larger range than those in Fig.~\ref{fig:side:b}. It indicates that the two-dimensional features of each identity become more compact through the regularization using attribute-aware loss. We can also observe that features for identities of the same age in Fig.~\ref{fig:side:b} are closer to each other than in Fig.~\ref{fig:side:a}, verifying the second property of the attribute-aware loss.

\begin{figure*}[t] 
	\centering
	\includegraphics[width=\textwidth]{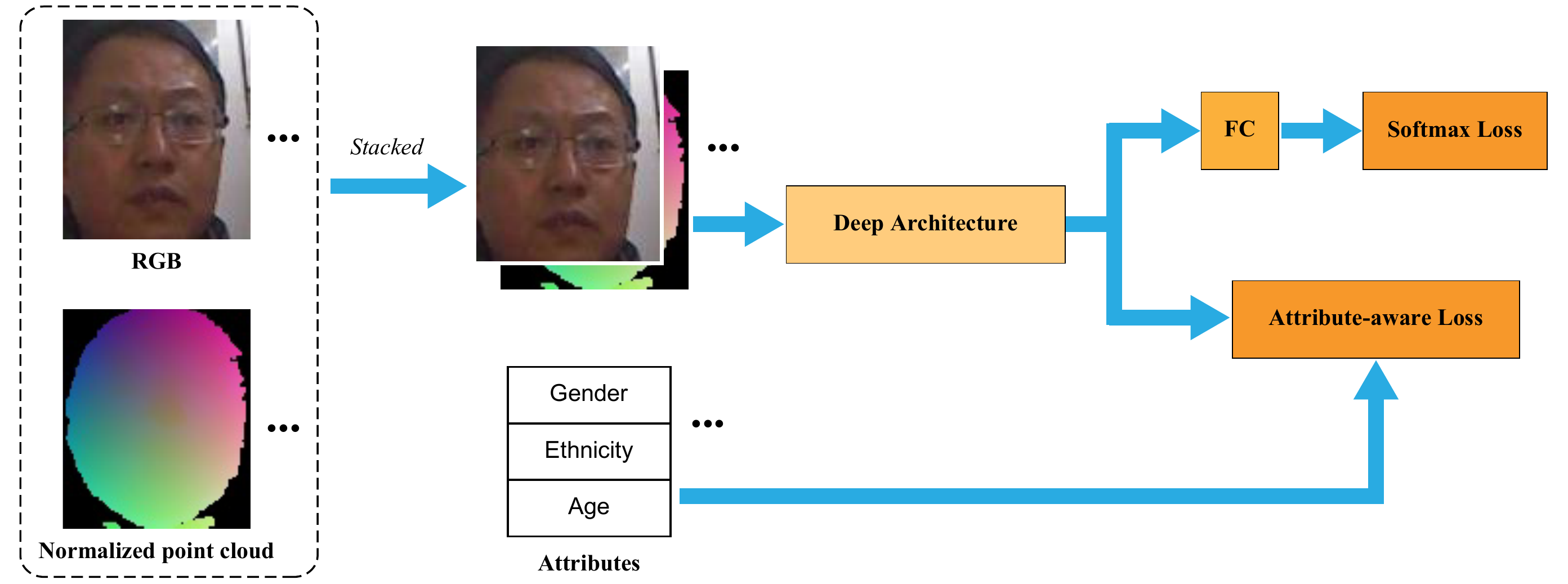}
	\caption{The framework of our approach. Facial images stacked with normalized point clouds and fed into a deep convolutional neutral network. The output features of the CNN are directly fed into a fully connected layer and afterwards the softmax loss layer. They are also fed into the attribute-aware loss layer together with their corresponding attribute vectors.}
	\label{fig:overview}
\end{figure*}

\subsection{Training with the Attribute-Aware Loss }
\label{subsec:AttributeAwareLossTraining}
Similar to~\cite{wen2016discriminative}, our attribute-aware loss in Eq.~\eqref{eq:AA_loss} is an auxiliary supervision signal, which can be combined with any variant of softmax loss. For example, it can be combined with the classical softmax to derive a  loss function
\begin{equation}
	L = L_s + \lambda L_a,
	\label{eq:TotalLoss}
\end{equation}
where $\lambda$ is a user-specified weight to balance the two loss terms.
In the following, we provide the details of mini-batch training with the loss function $L$. Each input mini-batch $\mathcal{B}$ consists of $M$ facial data $\{\bfx_{b_k}\}_{k=1}^M$ where $1 \leq b_k \leq N$, as well as their identity labels $\{y_{b_k}\}_{k=1}^{M}$ and attributes $\{\bfp_{b_k}\}_{k=1}^M$. These data are fed to the CNN, the softmax loss layer, and the attribute-aware loss layer respectively, as illustrated in Fig.~\ref{fig:overview}. 
The addition of an attribute-aware loss layer only introduces a slight overhead for the model size during the training phase, as it contains only one parameter matrix with  $K\times H$ parameters, where $K$ and $H$ are the dimensions of the deep facial feature and the attribute vector, respectively. In our implementation, $K = 512$ and $H = 3$, meaning that we only need 1536 additional parameters. In comparison, the backbone network that extracts deep features, which is a 28-layer ResNet~\cite{he2016deep}, requires about 0.3M parameters. Thus, the overhead from the attribute-aware loss layer is almost negligible. Moreover, during the testing phase, the attribute-aware loss layer is not needed and induces no overhead.

All parameters in the CNN and two loss layers can be learned using standard stochastic gradient descent.
The gradients of $L_a$ with respect to $\bfG$ and $\bff_k$ are computed as:
\begin{align}
\frac{\partial{L_a}}{\partial{\bfG}} = & \sum_{\substack{b_i < b_j \\ d(\bfp_{b_i},\bfp_{b_j})<\tau}} \left[\bfG(\bfp_{b_i} - \bfp_{b_j}) - (\bff_{b_i} - \bff_{b_j})\right]\left(\bfp_{b_i} - \bfp_{b_j}\right)^{T}, \nonumber\\
\frac{\partial{L_a}}{\partial{\bff_{b_k}}} = & \sum_{\substack{j \neq k \\ d(\bfp_{b_k},\bfp_{b_j})<\tau}} \left[(\bff_{b_k} - \bff_{b_j}) - \bfG(\bfp_{b_k} - \bfp_{b_j})\right].
\end{align}
Algorithm~\ref{alg: joint_supervision} summarizes the learning details in the CNNs with joint supervision.
\begin{algorithm}
\caption{The jointly supervised learning algorithm.}
\label{alg: joint_supervision}
\begin{algorithmic}[1]
\Require Training facial data $\{\bfx_i\}$ with identity labels $\{y_i\}$ and attributes $\{\bfp_i\}$; initial parameters $\Theta$ for the CNN; parameters $\bfW$, $\bfb$ for the softmax loss layer and $\bfG$ for the attribute-aware loss layer; balancing weight $\lambda$; learning rate $\alpha$; iteration count $t = 0$. 
\Ensure Trained parameters $\Theta$ for the CNN.
\While {not converge}
\State Compute forward by $L^t = L_s^t + L_a^t$.
\State Randomly choose a mini-batch $\mathcal{B}_t$.
\For {sample $i \in \mathcal{B}_t$}
\State Compute backward by $\displaystyle\frac{\partial{L}^{t}}{\partial{\bff}_{i}^{t}} = \displaystyle\frac{\partial{L}_{s}^{t}}{\partial{\bff}_{i}^{t}} + \lambda\frac{\partial{L}_{a}^{t}}{\partial{\bff}_{i}^{t}}$
\EndFor
\State Update $\mathbf{W}$:~~ $\bfW^{t+1} = \bfW^{t} - \displaystyle\frac{\alpha^{t}}{M}\displaystyle\frac{\partial{L}_{s}^{t}}{\partial{\bfW}^{t}}$.\vspace*{0.5em}
\State Update $\bfb$:~~  $\bfb^{t+1} = \bfb^{t} - \displaystyle\frac{\alpha^{t}}{M}\displaystyle\frac{\partial{L}_{s}^{t}}{\partial{\bfb}^{t}}$.\vspace*{0.5em}
\State Update $\bfG$:~~  $\bfG^{t+1} = \bfG^{t} - \displaystyle\frac{\alpha^{t}}{M}\displaystyle\frac{\partial{L}_{a}^{t}}{\partial{\bfG}^{t}}$.\vspace*{0.5em}
\State Update $\Theta$:~~ $\Theta^{t+1} = \Theta^{t} - \displaystyle\frac{\alpha^{t}}{M}\sum_{i=1}^{M}\displaystyle\frac{\partial{L}^{t}}{\partial{\bff}_{i}^{t}}\cdot\frac{\partial{\bff}_{i}^{t}}{\partial{\Theta}^{t}}$.\vspace*{0.2em}
\State Increment iteration count $t$.
\EndWhile
\end{algorithmic}
\end{algorithm}

\begin{figure*}[t] 
	\centering
	\includegraphics[width=\textwidth]{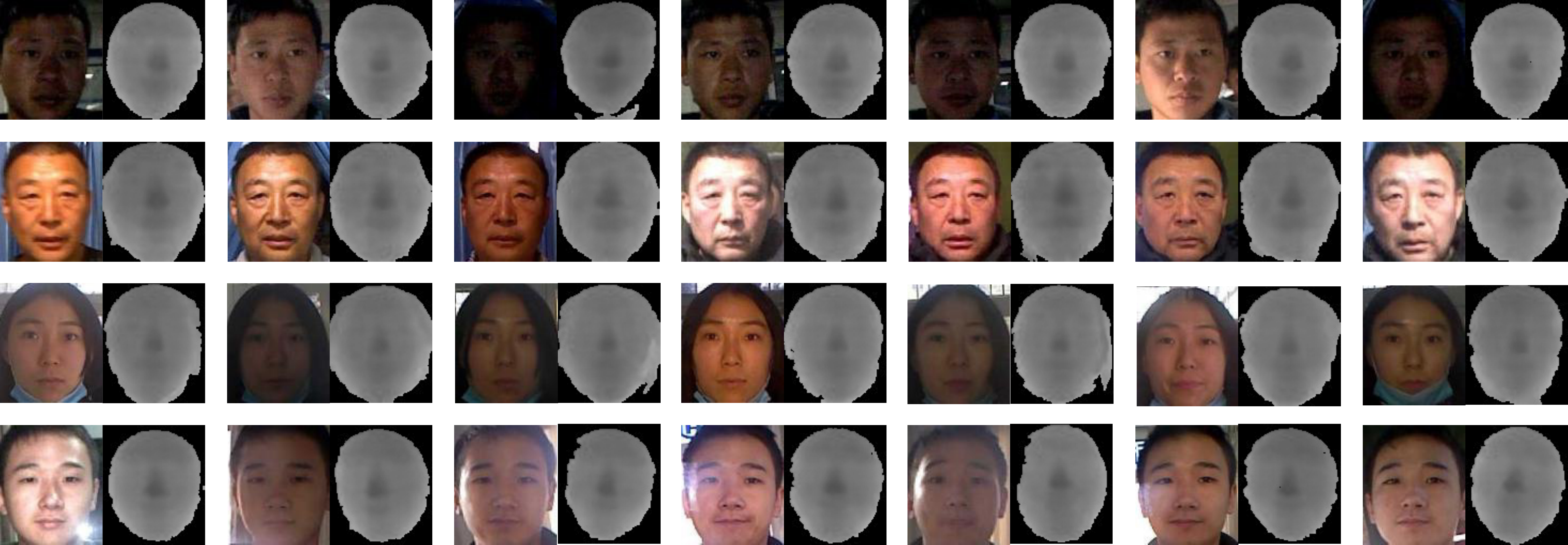}
	\caption{Samples from our large-scale RGB-D dataset. In each row, we show RGB-D data of one individual captured by PrimeSense camera at different locations and times and under different lighting conditions.}
	\label{fig:private_data}
\end{figure*}

\subsection{Training with RGB-D facial data}
In this paper, we use RGB-D facial images as the training data, to improve robustness to illumination conditions compared with RGB facial images. The RGB-D data are collected using low-cost sensors such as PrimeSense.
Using the RGB part of a facial image, we first detect the face region and five landmarks (the eyes, the nose, and two corners of the mouth) using MTCNN~\cite{zhang2016joint}. The face is then cropped to $112 \times 96$ by  similarity transformation, and each RGB color component is normalized from the range $[0, 255]$ into $[-1, 1]$. Afterward, we extract a face region from the corresponding depth image by transferring the RGB face region. Similar to~\cite{kim2017deep, zulqarnain2018learning}, we find the nose tip and crop the point cloud in the face region within an empirically set radius of $90$mm. Then we move the center of the cropped facial scan to $(0, 0, z_{opt})$ and reproject it onto a 2D image plane to generate a new depth map of size $112 \times 96$. The value $z_{opt}$ is chosen to enlarge the projection of facial scans onto the image plane as much as possible. Following~\cite{hernandez2015near}, we compute the depth of each pixel with bilinear interpolation. Using this depth map, we generate a new point cloud under the camera coordinate system. Each point $(v_x, v_y, v_z)$ is further normalized as:
\begin{eqnarray}
\overline{v}_x &=&(2v_x - \maxx - \minx)/(\maxx-\minx), \nonumber \\
\overline{v}_y &=&(2v_y - \maxy-\miny)/(\maxy-\miny),  \\
\overline{v}_z &=& (2v_z - \maxz-\minz)/(\maxz-\minz),\nonumber
\end{eqnarray}
where $(\minx, \miny, \minz)$ and $(\maxx, \maxy, \maxz)$ are the minimum and maximum $x$-, $y$- and $z$-coordinate values among all points, respectively. Augmenting the RGB face region with its normalized point cloud, we obtain a six-channel image with values in $[-1, 1]^6$, which is fed into the deep neutral network. Some RGB facial images and their normalized point clouds are shown in Fig.~\ref{fig:overview}.
\section{Experimental Results}
We conduct extensive experiments to evaluate the effectiveness of our approach. We first test our RGB-D face recognition approach on a large-scale private dataset (Secs.~\ref{subsec: param}, \ref{subsec:private} and \ref{subsec: fusion_scheme}) as well as some public datasets (Sec~\ref{subsec: public}). Then we compare our attribute-aware loss with other methods that utilize attributes for face recognition, using some public RGB datasets (Sec~\ref{subsec: fused}). 

\subsection{Implementation}
\label{subsec: details}
\paraheading{Our RGB-D dataset} 
We construct an RGB-D facial dataset that is captured by PrimeSense camera and contains more than 1.3M RGB-D images of 110K identities, where each identity has at least seven RGB images and their corresponding depth images. Most subjects are captured in the front of the camera with a neutral expression, and the multiple images of each subject are captured at different times and under different lighting conditions. Some samples from this RGB-D facial dataset are shown in Fig.~\ref{fig:private_data}. We also record their attributes including age, gender, and ethnicity. Compared with the datasets used for RGB-D face recognition in previous work~\cite{LeeCTL16, HCSC18}, our dataset contains a much larger number of identities, enabling us to evaluate the effectiveness of our approach in a real-world setting. 

\begin{table}[tbp]
	\renewcommand{\arraystretch}{1.3}
	\caption{The age and gender distribution of the two training data sets we construct, each including 60K individuals.}
	\label{tab:age-gender structure}
  \centering
  \begin{threeparttable}
    \begin{tabular}{ccccc}
    \toprule
    \multirow{2}{*}{Age Group}& \multicolumn{2}{c}{\trainsetone{}}& \multicolumn{2}{c}{\trainsettwo}\cr
	\cmidrule(lr){2-3} \cmidrule(lr){4-5}
	&\#male &\#female &\#male &\#female\cr
    \midrule
    (0,25] &6613 &6613 &4000 &2231\cr
    [26,35] &6613 &6613 &16504 &6245\cr
    [36,45] &6613 &6613 &12558 &4994\cr
    [46,55] &6612 &6612 &6612 &3688\cr
    [56,65] &3530 &1889 &1648 &108\cr
    [66,100]  &1420 &259 &1325 &87\cr
    \bottomrule
    \end{tabular}
    \end{threeparttable}
\end{table}

\paraheading{Implementation details} 
All CNN models are implemented using the Caffe library~\cite{jia2014caffe} with our modifications. Our CNN models are based on the same architecture as~\cite{wen2016discriminative}, using a 28-layer ResNet~\cite{he2016deep}. We train the models using stochastic gradient descent with different loss functions on RGB data, depth data, and their combination, respectively. All CNN models are trained with a batch size of 200 on two GPUs (TITAN XP). The learning rate begins at 0.1, and is divided by ten after 40K and 60K iterations, respectively. The training ends at 70K iterations. The facial data are horizontally flipped for data augmentation. During testing, we extract 512-dimensional deep features from the output of the first fully connected layer. For each test data, we concatenate its 512-dimensional features and its horizontally flipped 512-dimensional features as the final 1024-dimensional representation. In face verification and identification, the similarity between two features is computed using their cosine distance.

\begin{figure*}[t] 
	\centering
	\includegraphics[width=0.8\textwidth]{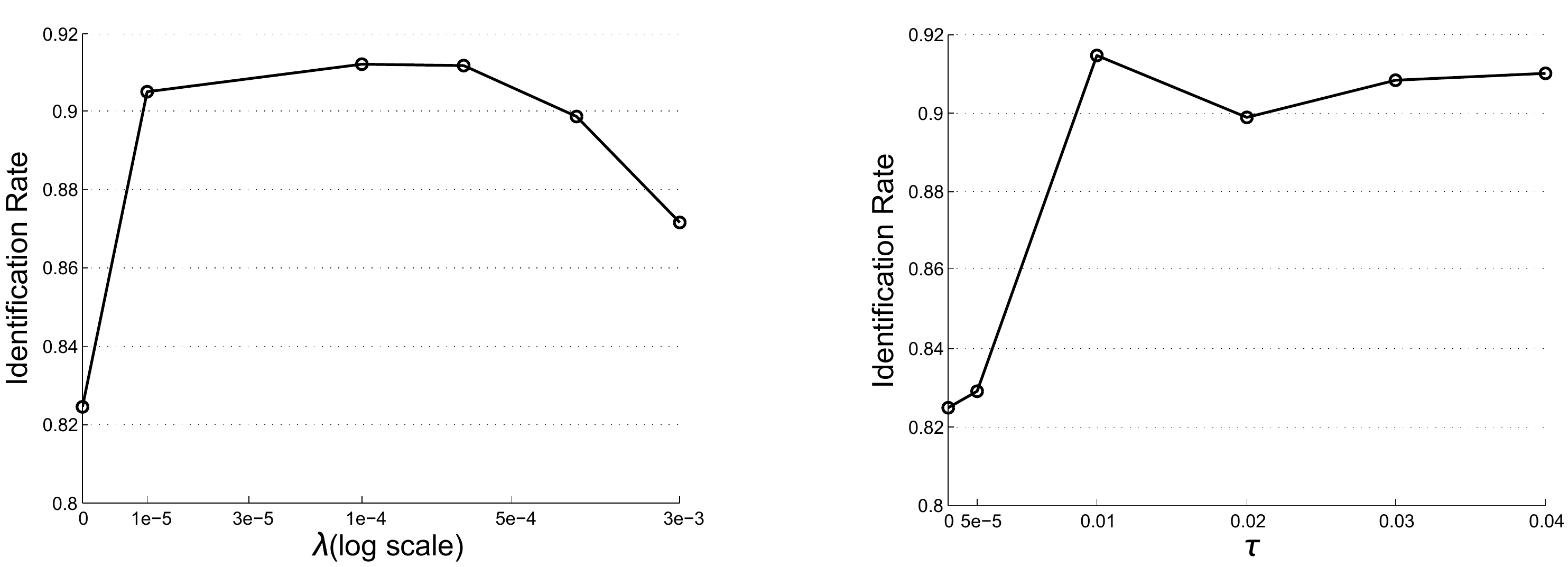}
	\caption{Rank-1 face identification rates of our approach on the RGB part of our constructed test data. Left: results achieved by fixing $\tau=0.02$ in Eq.~\eqref{eq:AA_loss} and varying $\lambda$ in Eq.~\eqref{eq:TotalLoss}. Right: results by fixing $\lambda=0.001$ and varying $\tau$.}
	\label{fig:param}
\end{figure*}

\begin{table*}[t]
	\renewcommand{\arraystretch}{1.3}
	\centering
	\begin{threeparttable}
		\caption{Comparison of face identification rates (\%) using two training datasets with four different loss functions.}
		\label{tab:performance_comparison}
		\begin{tabular}{ccccccc}
			\toprule
			\multirow{2}{*}{Loss functions}&
			\multicolumn{3}{c}{\trainsetone{}}&\multicolumn{3}{c}{\trainsettwo{}}\cr
			\cmidrule(lr){2-4} \cmidrule(lr){5-7}
			&RGB &Depth &RGB + Depth &RGB &Depth &RGB + Depth\cr
			\midrule
			(a) softmax &83.12 &83.83 &88.93 &82.46 &82.67 &86.59\cr
			(b) softmax + attribute-aware &92.24 &90.70 &96.26 &92.00 &90.32 &96.16\cr
			(c) softmax + center &96.56 &88.49 &98.69 &95.89 &87.85 &98.47\cr
			(d) softmax + center + attribute-aware &\textbf{96.78} &\textbf{93.29} &\textbf{98.99} &\textbf{96.28} & \textbf{92.86} & \textbf{98.75}\cr
			\bottomrule
		\end{tabular}
	\end{threeparttable}
\end{table*}

\subsection{Experiments on the Parameters $\lambda$ and $\tau$}
\label{subsec: param}
The parameter $\lambda$ controls the importance of attribute-aware loss $L_a$, while the parameter $\tau$ decides whether a pair of attribute vectors are close enough to be considered in the attribute-aware loss. Since both of them are important for our loss function, we conduct two experiments to illustrate how $\lambda$ and $\tau$ influence the face recognition performance. We first construct a training set and a test set by sampling the whole dataset. This training set (\trainsetone{}) includes about 0.88M RGB-D images of 60K identities, with 91\% Caucasians and 9\% Asians. Within the training set there are balanced distributions of age and gender, as shown in Tab.~\ref{tab:age-gender structure}. The test set includes about 0.22M RGB-D images of 20K identities. The first available neutral image of each identity in the test set is placed in the gallery, and the remaining images are used as probes. We select gender, ethnicity, and age as the attributes for training the model. For the gender attribute, we use 1 to indicate male and -1 for females. For ethnicity, since our dataset only contains Asians and Caucasians, we use 1 for Asians and -1 represents Caucasians. For age, we first truncate the age value at 100, and then linearly map it from the range $[0, 100]$ into $[-1, 1]$. In this way, we represent the attributes as 3-dimensional vector $\bfp_i = (p_i^g, p_i^e, p_i^a)$ in $[-1, 1]^3$, where the superscripts $g$, $e$ and $a$ indicate gender, ethnicity and age, respectively.

To demonstrate the effectiveness and sensitivity of the two parameters, we train our models jointly supervised with the softmax loss and the attribute-aware loss on the RGB part of the constructed dataset. In the first experiment, we fix $\tau=0.02$ and vary $\lambda$ from 0 to 0.003 to learn different models. Performance on the closed-set identification task is the classical evaluation criteria for face recognition. We show the rank-1 identification rates of these models on our test set in Fig.~\ref{fig:param}(left). We can see that our attribute-aware loss can greatly improve the face recognition performance, especially when $\lambda$ is in the range $[10^{-5}, 10^{-3}]$. In the second experiment, we fix $\lambda = 0.001$ and vary $\tau$ from 0 to 0.04. The corresponding rank-1 identification rates on our test set are shown in Fig.~\ref{fig:param}(right). It can be observed that the identification rates remain stable for $\tau \in [0.01, 0.04]$. Within this range, there are between 150 and 630 pairs of similar attribute vectors with different identities in one batch. In practice, we prefer to select a small value for $\tau$ due to its lower computational cost.

\subsection{Experiments on Our RGB-D Dataset}
\label{subsec:private}
\paraheading{Training \& test sets}
To better verify whether the attributes and the depth data can improve the face recognition performance, we construct another training set (\trainsettwo{}) from the whole dataset. This training set also includes about 0.88M RGB-D images of 60K identities, with the same distribution of ethnicity but less balanced distributions of age and gender compared with \trainsetone{} (see Tab.~\ref{tab:age-gender structure}). We use the probe set and gallery set in Sec~\ref{subsec: param} as our test set. In this experiment, we train and test our models on RGB, depth and their combination, respectively.

\paraheading{Training CNNs}
For a fair comparison, we train the CNN model with four loss functions respectively: (a) softmax loss; (b) softmax loss combined with attribute-aware loss; (c) softmax loss combined with center loss~\cite{wen2016discriminative}; (d) softmax loss combined with center loss and attribute-aware loss. We use models (a) and (c) as baselines, to demonstrate the effect of adding attribute-aware loss into the overall loss function. The weights of center loss and attribute-aware loss are set to 0.001 and 0.0001, respectively. The margin $\tau$ used in Eq.~\eqref{eq:AA_loss} is set to 0.01 based on the analysis in Sec.~\ref{subsec: param}. For model (d), we use the center loss to fine-tune the CNN model trained by the attribute-aware loss. For fine-tuning, the learning rate begins at 0.01, and is divided by 10 after 20K and 30K iterations, respectively. The fine-tuning ends at 35K iterations.

\paraheading{Results \& Analysis}
The rank-1 identification rates are shown in Tab.~\ref{tab:performance_comparison}. It shows that training with both RGB and depth can achieve better face recognition performance than with RGB or depth alone. As the depth image captured by the PrimeSense camera is of lower quality than the corresponding RGB image, recognition features trained with depth images are less discriminative than those with RGB images. On the other hand, depth information can be a helpful complement to RGB for improving face recognition accuracy. 

The benefit of incorporating depth information into face recognition is further shown in Fig.~\ref{fig:robust}, where we show some probes that cannot be identified correctly at rank one using the model trained with softmax loss on RGB data, but can be identified correctly when the model is trained with the same loss on RGB-D data. This is due to the very different lighting conditions between the probes and their corresponding sample in the gallery, which makes it difficult to perform face recognition using only RGB data. On the other hand, depth data is more robust to lighting than RGB data, which helps to refine the learned recognition features. Indeed, the use of depth data improves the ranking of the probes and reduces the cosine distance between the features of each probe and its gallery sample.

\begin{figure}[!t] 
	\centering
	\includegraphics[width=\columnwidth]{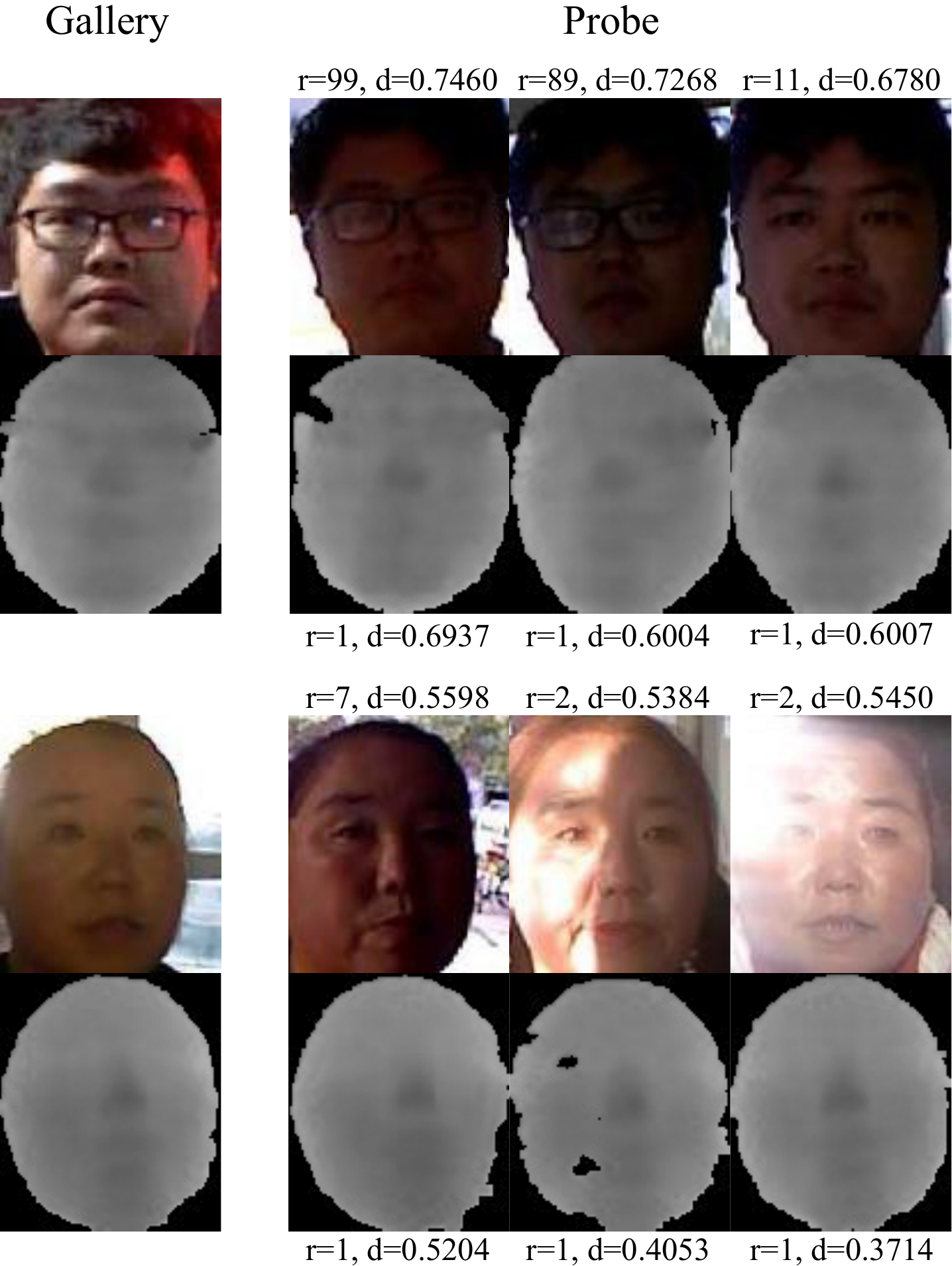}
	\caption{RGB-D data of two individuals in our gallery and their corresponding three RGB-D images in our probe set. These probes cannot be identified correctly at rank one without depth data. On top of each RGB probe, we show its ranking (r) in RGB-only probing and the cosine distance (d) between the features extracted from the RGB probe and its gallery sample. Below each depth probe, we show the ranking in RGB-D probing and the cosine distance between features extracted from the RGB-D probe and its gallery sample. }
	\label{fig:robust}
\end{figure}

\begin{figure*}[t] 
	\centering
	\includegraphics[width=\textwidth]{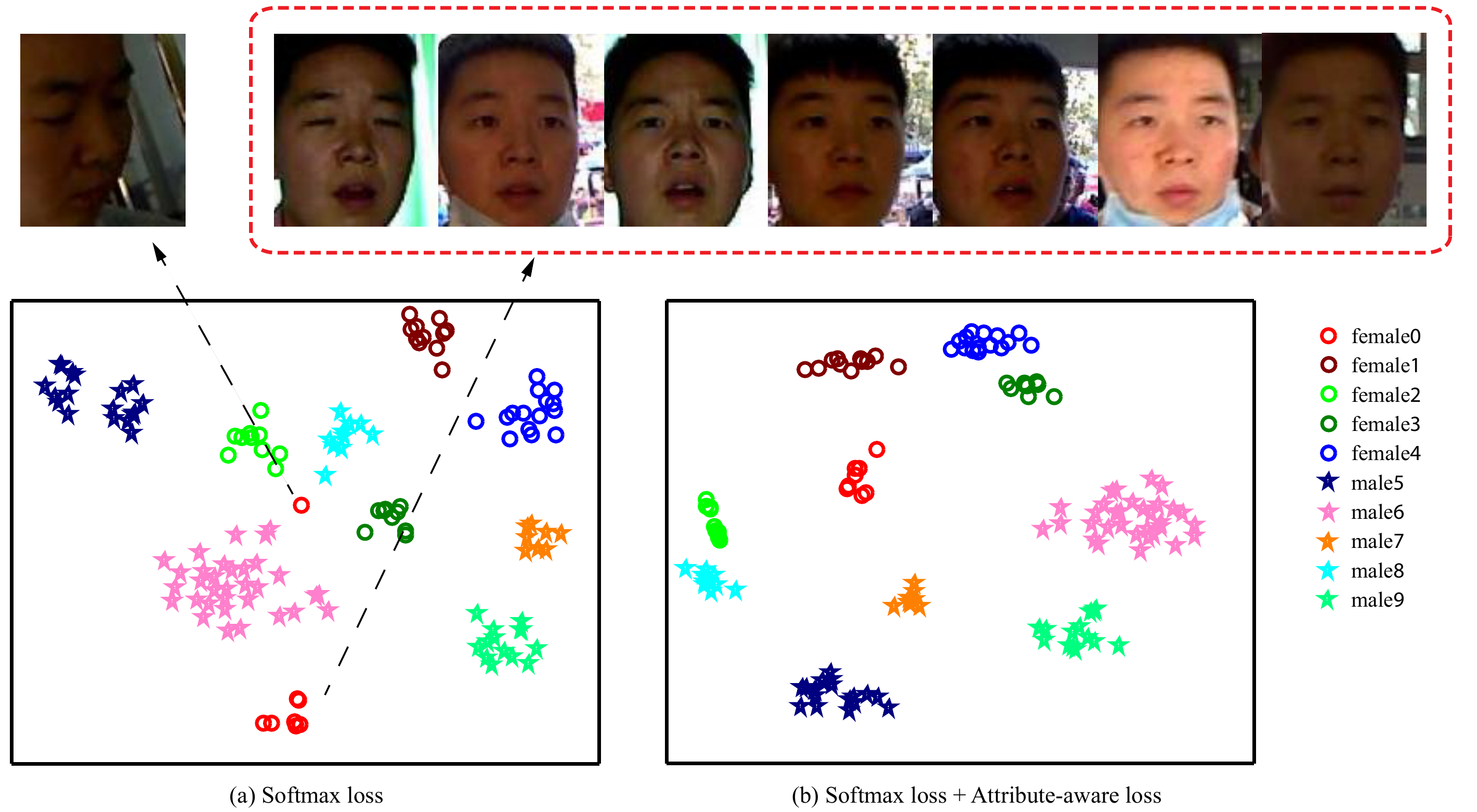}
	\caption{The distribution of learned features supervised by softmax loss (a) and supervised jointly by softmax loss and attribute-aware loss (b) in two-dimensional space after dimension reduction using t-SNE~\cite{Maaten08vd}. There are ten feature clusters with different color, five males (pentacle) and five females (circle). All of color images of the first female (red circle) are shown under the learned features.} 
	\label{fig:gender}
\end{figure*}

In Fig.~\ref{fig:gender}, we show ten clusters of features trained on RGB data with softmax loss (see Fig.~\ref{fig:gender}(a)) and with the combined softmax and attribute-aware loss (see Fig.~\ref{fig:gender}(b)), respectively. The results are visualized in 2D using dimensionality reduction. Without the attribute-aware loss, the five male feature clusters are interlaced with the female feature clusters. In comparison, with an attribute-aware loss there is a more clear separation between male and female feature clusters, and the features within each cluster are closer to each other. In particular, there is one photo of a female subject with large head pose, which looks significantly different from other photos of the same subject. 
Without attribute-aware loss, this photo is mapped to a feature far away from the feature cluster of the subject (shown as red circles) and becomes an outlier. Using attribute-aware loss, all photo of this female subject is mapped to a tight cluster. 

Tab.~\ref{tab:performance_comparison} also shows that the addition of attribute-aware loss into the overall loss function results in more accurate face identification results compared with the two baseline models without attribute-aware loss.
Especially, the model with combined softmax and attribute-aware loss outperforms the model with only softmax loss by a significant margin. 
Moreover, by comparing the results on \trainsetone{} and \trainsettwo{}, we can observe that the attribute-aware loss improves the robustness of the results when the training datasets are not evenly sampled, because it uses the difference of attributes as a proxy to take the sampling bias into account.

\begin{table}[tbp]
	\renewcommand{\arraystretch}{1.3}
		\caption{Comparison of face identification rates (\%) using \trainsetone{} and \trainsettwo{} with detected attribute features.}
		\label{tab:detected_attributes}
		\centering
		\begin{tabular}{lcccc}
			\hline
			Training set &RGB &Depth &RGB + Depth\\ \hline      
			\trainsetone{} &87.40 &88.14 &92.84\\ \hline
			\trainsettwo{} &87.04 &89.97 &92.42\\ \hline
		\end{tabular}
\end{table}

\subsection{Experiments on Our RGB-D Dataset with Detected Attribute Features}
\label{subsec: detected_attributes}
Since our RGB-D facial dataset records three non-facial attributes for each subject, we can use this dataset to train a network that detects attribute features from facial images. We also evaluate the effectiveness of our approach with such detected attributes, by replacing the ground-truth attributes with detected attribute features in the attribute-aware loss layer. Following~\cite{hu2017attribute}, we adopt the architecture of Lighten CNN~\cite{xiang2015light} as the attribute detection network. We train this CNN model with a multi-task loss function MOON~\cite{rudd2016moon} on the RGB data of \trainsetone{}. The learning rate begins at 0.01, and is divided by ten after 40K and 60K iterations, respectively. The training ends at 70K iterations. We then extract 256-dimensional deep attribute features from the penultimate fully connected layer of the attribute detection network, to feed into the attribute-aware loss layer. Using these detected attribute features, we retrain our face recognition network jointly supervised with softmax loss and attribute-aware loss on both \trainsetone{} and \trainsettwo{}. The training details are the same as in Sec.~\ref{subsec:private}. The rank-1 identification rates are shown in Tab~\ref{tab:detected_attributes}. We can see that using the attribute-ware loss layer with detected attribute features also improve the face recognition accuracy, although not as much as the ground-truth attributes. It shows that our approach can also be applied in scenarios where ground-truth attributes are not available in the training data.

\subsection{Experiments on Our RGB-D Dataset with Controlled Attributes}
\begin{table}[tbp]
	\renewcommand{\arraystretch}{1.3}
	\centering
	\caption{Comparison of face identification rates (\%) by training on the RGB-D data of \trainsetone{} with three controlled attributes. Here ``a'', ``e'', and ``g'' indicate age, ethnicity, and gender, respectively.}
	\label{tab:individual_attributes}
	\centering
	\begin{tabular}{lcccc}
		\hline
		Loss function &e + a & g + a &g + e\\ \hline      
		softmax + attribute-aware &95.52 &96.05 &94.16\\ \hline
	\end{tabular}
\end{table}

In Sec.~\ref{subsec:private}, we use three attributes in the attribute-ware loss layer to achieve good performance in the test phase. To further understand the importance of each individual attribute, three new kinds of controlled attributes are designed respectively: (a) ethnicity and age $(p_i^e, p_i^a)$; (b) gender and age $(p_i^g, p_i^a)$; (c) gender and ethnicity $(p_i^g, p_i^e)$. We ignore one of the ground-truth attributes, in turn, to observe its corresponding performance change. The dimension of attribute vectors becomes two in the attribute-aware loss layer now.
We train the CNN model with the three kinds of controlled attributes on the RGB-D data of \trainsetone{} respectively. The training details are the same as in Sec.~\ref{subsec:private}. The margin $\tau$ in Eq.~\eqref{eq:AA_loss} is still set to 0.01 for the three cases. Since most identities in \trainsetone{} are Caucasians, case (a) and case (c) can have more selected pairs in every mini-batch than case (b) when training the CNN models. For the sake of fairness, we also require that the number of selected pairs should not exceed a given threshold in every mini-batch. We set this threshold to 350, which makes sure that the three cases can have a similar number of pairs. The rank-1 identification rates are shown in Tab.~\ref{tab:individual_attributes}. We can see the performance of case (c) drops the most, followed by the case (a) and case (b). The result indicates that age is more important than gender in learning strong discriminative features. This may be due to the fact that age is a variable with a larger range of values, which potentially provides more information.

\begin{figure}[!t] 
	\centering
	\includegraphics[width=\columnwidth]{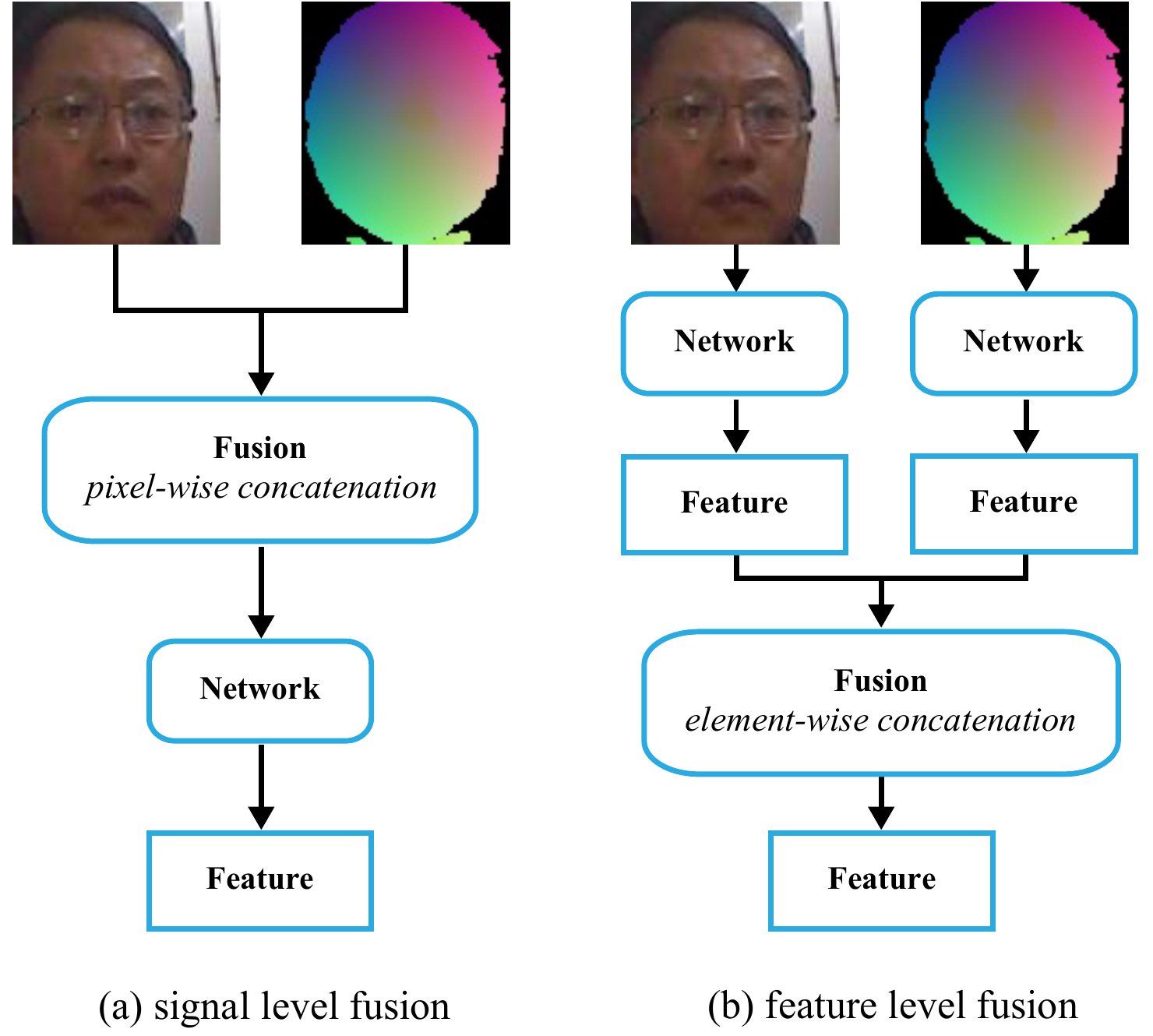}
	\caption{The diagrams of the two representative fusion schemes for RGB-D face recognition: (a) signal level fusion, (b) feature level fusion.}
	\label{fig:fusion}
\end{figure}

\begin{figure*}[t] 
	\centering
	\includegraphics[width=\textwidth]{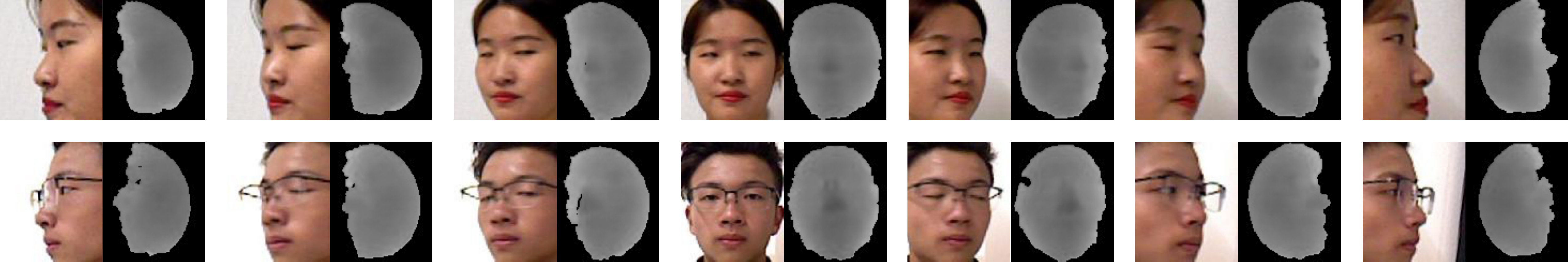}
	\caption{Samples from our small-scale RGB-D dataset under different poses. In each row, we show RGB-D data of one individual captured by PrimeSense under the pose with yaw ranging from left $90^\circ$ to right $90^\circ$.}
	\label{fig:multiview}
\end{figure*}

\subsection{Experiments on RGB-D Fusion Schemes}
\label{subsec: fusion_scheme}
In Sec~\ref{subsec:private}, we concatenate RGB and depth into a six-channel data as input to the 28-layer ResNet, which is a signal-level fusion scheme for RGB-D face recognition (see Fig.~\ref{fig:fusion}(a)). It is worth noting that there are other ways to fuse RGB and depth information. Similar to~\cite{cui2018rgb}, we also explore a feature-level fusion scheme (see Fig.~\ref{fig:fusion}(b)), which fuses the features extracted separately from RGB and depth modality networks and then feeds them into the classification layers. We use the same 28-layer ResNet architecture for both the RGB and depth data and test three different types of the feature-level fusion schemes, which are implemented at the front, the middle and the back of ResNet-28 respectively. More concretely, their fusion operation is performed ahead of the fourth convolution layer, the eighth convolution layer and the first fully connected layer of ResNet-28, respectively. The performance of these feature-level fusion schemes is then compared with the signal-level fusion scheme.
Since the feature-level fusion schemes involve more parameters, we need to set a smaller batch size for training, which in turn requires more training iterations. In our experiments, we set the batch size to 100, and train for 180K iterations. Each feature-level fusion scheme takes about 14 hours to train, whereas the signal-level fusion scheme takes only about 10 hours.
Tab.~\ref{tab:fusion_scheme} compares the performance of different fusion schemes. It can be observed that only the feature-level scheme implemented at the middle outperforms the signal-level scheme by a small margin, although with a 40\% increase of training time. The feature-level scheme implemented at the back performs much worse than all other schemes, especially for the softmax loss. One possible reason is that the features are concatenated using a fully connected layer followed by a loss function, which is not enough for feature fusion. As pointed out in~\cite{cui2018rgb}, the fusion of multiple feature descriptors can improve face recognition rates. However, we need to carefully design fusion schemes to achieve better performance.

\begin{table}[tbp]
	\renewcommand{\arraystretch}{1.3}
  \caption{Comparison of face identification rates (\%) and training time (hours) using the RGB-D data of \trainsetone{} with different loss functions and fusion schemes. Here ``front'', ``mid'' and ``back'' indicate fusion operations ahead of the fourth convolution layer, the eighth convolution layer and the first fully connected layer of ResNet-28~\cite{he2016deep}, respectively.}
  \label{tab:fusion_scheme}
  \centering
  \begin{threeparttable}
    \begin{tabular}{ccccc}
    \toprule
    \multirow{2}{*}{~}&\multirow{2}{*}{signal-level}&
    \multicolumn{3}{c}{feature-level}\cr
    \cmidrule(lr){3-5}
      &  & front &mid &back\cr
    \midrule
    softmax &88.93 &88.67 &89.44 &82.67 \cr
    softmax + attribute-aware &96.26 &95.70 &96.47 &94.90 \cr
    \midrule
    training time &  $\sim$10h & $\sim$14h & $\sim$14h & $\sim$14h \cr
    \bottomrule
    \end{tabular}
    \end{threeparttable}
\end{table}

\begin{table*}[t]
	\renewcommand{\arraystretch}{1.3}
\caption{Comparison of face identification rates (\%) using the RGB-D data of the small-scale training set under different poses.}
\label{tab:pose}
\centering
\begin{tabular}{lcccccc}
\hline
Loss functions &$-90^\circ$ &$(-90^\circ, -45^\circ)$ &$(-45^\circ, +45^\circ)$ &$(+45^\circ, +90^\circ)$ &$+90^\circ$\\ \hline      
softmax &62.65 &88.36 &98.59 &88.38 &50.12 \\ \hline
softmax + attribute-aware &\textbf{81.15} &\textbf{93.97} &\textbf{99.65} &\textbf{92.96} &\textbf{65.88}\\ \hline
\end{tabular}
\end{table*}

\subsection{Experiments on a Small-Scale RGB-D Dataset Under Pose Variation}
\label{subsec: pose_variation}
Most identities of the RGB-D dataset in Sec.~\ref{subsec: details} are with a frontal pose. To demonstrate the robustness of our method for larger poses, we construct another small-scale RGB-D facial dataset that is captured by PrimeSense camera in the same indoor scene and contains about 25k RGB-D images of 952 identities, where each identity also has corresponding attributes including gender, ethnicity, and age. The new dataset contains only young Asians with $30\%$ male and $70\%$ female, and covers pose with yaw ranging from left $90^\circ$ to right $90^\circ$. Some samples from this new dataset are shown in Fig.~\ref{fig:multiview}. This dataset is available at \url{http://staff.ustc.edu.cn/~juyong/RGBD_dataset.html}. We use the first 852 identities to construct a training set and the remaining 100 identities for testing. Similar to~\cite{hu2017attribute}, the first frontal image of each identity is placed in the gallery, with the remaining images used as probes.

We train the CNN model jointly supervised with softmax loss and attribute-aware loss on the RGB-D data of the training set. The architecture of the CNN model is the same as in Sec.~\ref{subsec:private}. The weight of the attribute-aware loss is set to 0.0001 and the distance margin $\tau$ is set to 0.001. The learning rate begins at 0.1, and is divided by ten after 1.1k and 1.7k iterations. The training ends at 2k iterations. The rank-1 identification rates are shown in Tab~\ref{tab:pose}. We can see that the model with the combined softmax and attribute-ware loss outperforms the model with only the softmax loss across the whole range of pose variation, with greater improvements for more extreme poses. This example shows that attribute-aware loss can enhance the robustness of facial features with respect to pose variation.

\subsection{Experiments on Public RGB-D Datasets}
Although our RGB-D face recognition model is trained with data captured by PrimeSense
cameras, it can also be applied to other RGB-D data with good generalization ability. In the following, we test our RGB-D face recognition model trained using \trainsetone{} on some public RGB-D datasets, and compare it with other RGB-D face recognition methods.

\begin{figure*}[t] 
	\centering
	\includegraphics[width=\textwidth]{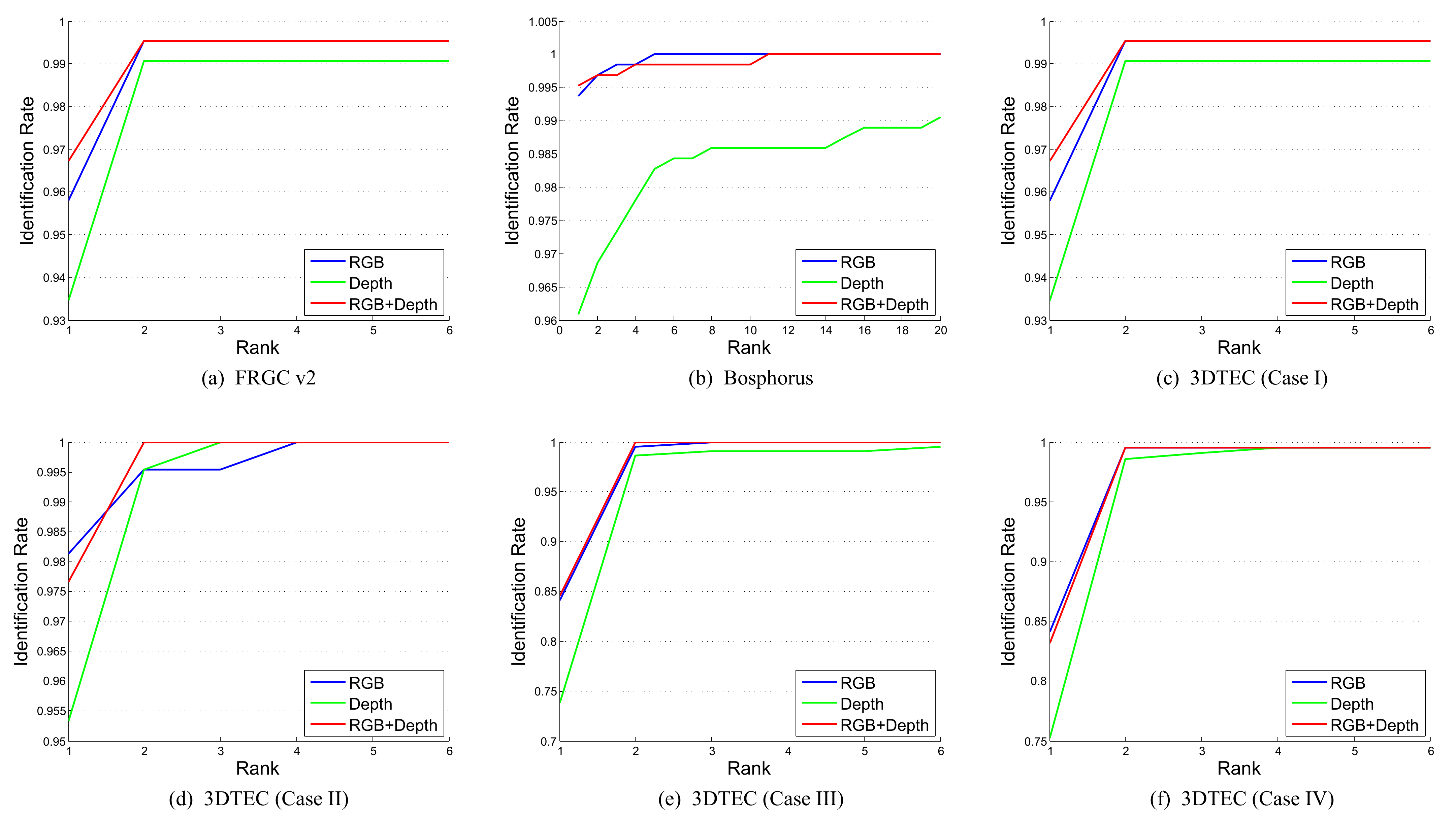}
	\caption{Evaluation of our method on three public data sets: FRGC v2~\cite{phillips2005overview}, Bosphorus~\cite{savran2008bosphorus} and 3D-TEC~\cite{vijayan2011twins}. For each test data set, we show cumulative match characteristic (CMC) curves of three models trained using RGB, depth, and RGB-D of \trainsetone{} respectively.}
	\label{fig:rates}
\end{figure*}

\begin{table*}[t]
	\renewcommand{\arraystretch}{1.3}
	 \caption{Comparison of Rank-1 identification rates (\%) with state-of-the-art methods on public data sets. Except for our method, the result data are taken from the respective papers~\cite{zulqarnain2018learning,kim2017deep,faltemier2008region,li2014expression,GoswamiVS14,LiXMLK16}. The best result for each data set is highlighted in bold font.}
	\label{tab:performance_public_rgbd}
  \centering
  \begin{threeparttable}
    \begin{tabular}{ccccccccc}
    \toprule
    \multirow{2}{*}{Method}& \multicolumn{1}{c}{\multirow{2}{*}{Training data}}& \multicolumn{1}{c}{\multirow{2}{*}{Modality}} &\multicolumn{1}{c}{\multirow{2}{*}{FRGC}}& \multicolumn{1}{c}{\multirow{2}{*}{Bosphorus}}& \multicolumn{4}{c}{3DTEC}\cr
	\cmidrule(lr){6-9}
	&&&&&Case \Rnum{1} &Case \Rnum{2} &Case \Rnum{3} &Case \Rnum{4}\cr
    \midrule
    GoogleNet~\cite{szegedy2015going} &Private~\cite{zulqarnain2018learning} &RGB &21.51 &63.44 &- &- &- &-\cr
    VGG-Face~\cite{ParkhiVZ15}  &Private~\cite{zulqarnain2018learning} &RGB &87.92 &96.39 &- &- &- &-\cr
    Ours      &\trainsetone{} &RGB &95.69 &96.08 &95.79 &98.13 &84.11 &84.11\cr
\hdashline
    Faltemier et al.~\cite{faltemier2008region} &FRGC~\cite{phillips2005overview} &Depth &- &- &94.40 &93.50 &72.40 &72.90\cr
    Li et al.~\cite{li2014expression} &Part of Bosphorus~\cite{savran2008bosphorus} &Depth &96.30 &95.40 &93.90 &96.30 &\textbf{90.70} &91.60\cr
    Kim et al.~\cite{kim2017deep} &FRGC~\cite{phillips2005overview} &Depth &- &99.20 &94.80 &94.80 &81.30 &79.90\cr
    FR3DNet~\cite{zulqarnain2018learning} &Private~\cite{zulqarnain2018learning} &Depth &97.06 &96.18 &- &- &- &-\cr
    Ours &\trainsetone{} &Depth &97.45 &99.37 &93.46 &95.33 &73.83 &75.23\cr
\hdashline
    Goswami et al.~\cite{GoswamiVS14} &None &RGB+Depth &- &- &95.80 &94.30 &90.10 &\textbf{92.40}\cr
    Li et al.~\cite{LiXMLK16} &Part of FRGC~\cite{phillips2005overview} &RGB+Depth &95.20 &99.40 &- &- &- &-\cr
    Ours & \trainsetone{} &RGB+Depth &\textbf{98.52} &\textbf{99.52} &\textbf{96.73} &\textbf{97.66} &84.58 &83.18\cr
    \bottomrule
    \end{tabular}
    \end{threeparttable}
\end{table*}

\paraheading{Performance on FRGCv2}
The FRGCv2 dataset~\cite{phillips2005overview} contains 4,950 scans and texture images of 557 identities captured across multiple sessions and with two kinds of expressions (e.g., neutral and smile). To test face identification performance, we place the first available neutral scan and the texture image of each identity in the gallery, and use the remaining RGB-D images as probes. Tab.~\ref{tab:performance_public_rgbd} reports the rank-1 identification rates of different methods, while Fig.~\ref{fig:rates}(a) shows the cumulative match characteristic (CMC) curves. Our method performs the best under all three data modalities (RGB, Depth, and RGB + Depth).

\paraheading{Performance on Bosphorus}
The Bosphorus dataset~\cite{savran2008bosphorus} contains 4,666 scans including color and depth images from 105 identities with rich expression variations, poses, and occlusions. There are 743 scans that only include expression variations. In our experiment, 105 first neutral depth and color images from each identity are used as the gallery set, and 638 non-neutral scans are used as the probe set. We report the rank-1 identification rate in Tab.~\ref{tab:performance_public_rgbd} and show the CMC curves in Fig.~\ref{fig:rates}(b). Although our training data does not include the public data sets and our trained model is not fine-tuned on any gallery of public test data sets, our model can still achieve the best results on Bosphorus dataset.

\paraheading{Performance on 3D-TEC}
The 3D-TEC dataset~\cite{vijayan2011twins} contains 107 pairs of twins acquired using a Minolta VIVID 910 3D scanner under controlled illumination and background. We use the standard protocol for 3D-TEC (four scenarios: Case I,II,III, and IV) described in~\cite{vijayan2011twins}. We show the rank-1 identification rate in Tab.~\ref{tab:performance_public_rgbd} and the CMC curves in Figs.~\ref{fig:rates}(c)-(f). Cases III and IV are challenging scenarios where the system does not control the expressions of the subject in a gallery set of twins. Since the depth image of the training data for our models is captured by PrimeSense cameras and of low quality, our models do not work well for these scenarios that require a distinction between subtle differences in the depth data. This also explains why our RGB model performs even better than our RGB-D model.
\label{subsec: public}

\subsection{Comparison with Other Methods that Utilize Attributes}
\label{subsec: fused}
To demonstrate the effectiveness of our attribute-aware loss, we also compare our method with GTNN~\cite{hu2017attribute}, a state-of-the-art face recognition method utilizing attributes. GTNN fuses the facial recognition features (FRFs) and facial attribute features (FAFs) together via a nonlinear tensor-based fusion method, and the fused features are used for the final face recognition. Different from GTNN, our method needs attribute information only for the training stage but not the testing stage.

\paraheading{Training \& test sets} 
We train GTNN and our model on publicly available web-collected RGB datasets, including CASIA-Webface~\cite{yi2014learning} and CelebA~\cite{liu2015deep}. CASIA-Webface includes about 494K facial RGB images of 10,575 identities without facial attributes. CelebA is a large-scale facial attributes dataset with more than 200K celebrity RGB images of 10,177 identities, each with 40 attribute annotations. Same as~\cite{hu2017attribute}, we remove attributes that do not characterize a specific person and retain 13 attributes. The results are tested on the LFW dataset~\cite{huang2014labeled}, which contains 13,233 RGB images of 5,749 identities. During testing, LFW is divided into ten predefined subsets for cross-validation. We follow the standard \emph{Unrestricted, Labeled Outside Data Results} protocol for testing.

\paraheading{Training CNNs}
The input for GTNN consists of pre-extracted FRFs and FAFs. We train a CNN model (\modela{}) with softmax loss on CASIA-Webface to extract FRFs. We also train a CNN model (\modelb{}) with the loss function MOON~\cite{rudd2016moon} on CelebA to extract FAFs. Then we use \modela{} and \modelb{} to extract FRFs and FAFs on CASIA-Webface respectively. Therefore, we can train the GTNN (\modelc{}) with the softmax loss on those FRFs and FAFs. In order to compare with GTNN, we regard those FAFs as the attributes and train a CNN model (\modeld{}) with the combined softmax and attribute-aware loss on CASIA-Webface. \modela{}, \modelb{} and \modelc{} are trained according to the details given in~\cite{wen2016discriminative, hu2017attribute, rudd2016moon}. For training \modeld{}, the learning rate begins at 0.1, and is divided by ten after 16K and 24K iterations, respectively. The training ends at 28K iterations. The weight of the attribute-aware loss is 0.001 and the margin $\tau$ is set to 0.58, such that there are about 300 pairs of similar attribute vectors with different identities in one batch.

\paraheading{Results \& Analysis}
The face verification rates of four models on the LFW dataset are reported in Tab.~\ref{tab:fusion}. The GTNN~\cite{hu2017attribute} fusion method improves the verification rate compared with training with FRFs or FAFs alone, while our attribute-aware achieves the best accuracy among the four approaches.
Another thing to note is the greater efficiency of our approach. Training with our attribute-aware loss layer improves the accuracy of the resulting FRFs, with an almost negligible increase in the number of training parameters. In the testing phase, the attribute-aware loss layer is not needed and incurs no overhead. In comparison, GTNN fuses FRFs and FAFs into high-dimensional features via a tensor-based approach that involves a large number of training parameters. Although the number of parameters can be reduced by Tucker decomposition as shown in~\cite{hu2017attribute}, the rank of the tensor must be carefully chosen to achieve a good balance between accuracy and efficiency. Moreover, GTNN requires training the FRF model, the FAF model and the fusion model, whereas our approach only requires training a single FRF model.

\begin{table}[tbp]
\renewcommand{\arraystretch}{1.3}
\caption{Comparison with fusion method GTNN~\cite{hu2017attribute} on LFW data set.}
\label{tab:fusion}
\centering
\begin{tabular}{lccc}
\hline
Models &Descriptions &Acc.(\%)\\ \hline      
\modela{} &FRFs &97.43\\ \hline
\modelb{} &FAFs &73.92\\ \hline
\modelc{} &GTNN~\cite{hu2017attribute} fusion &97.78\\ \hline
\modeld{} &Our method &\textbf{98.42}\\ \hline
\end{tabular}
\end{table}
\section{Discussion \& Conclusion}

We have presented an attribute-aware loss function for CNN-based face recognition, which regularizes the distribution of learned recognition features with respect to additional attributes. The novel attribute-aware loss could help resolve the problem of uneven sampling in the training dataset, and improves the face recognition accuracy. Besides, we train an RGB-D based face recognition model using a large data set with over 100K identities. The experimental results demonstrate the effectiveness of our novel attribute-aware loss, and the good generalization ability of our trained RGB-D face recognition model.

Our work is the first method on using non-facial attributes which are invariant to the capture environments to regularize the face recognition feature mapping. In this work, although we only use gender, age, and ethnicity attribute for regularization, the experimental results still demonstrate big improvements. Other attributes could also be added for further study. Another point that needs to be further investigated is how the transformation formulation between recognition feature and attribute feature influences the final recognition performance. For example, we could consider replacing the linear transformation $\mathbf{G}$ with nonlinear transformation. Currently, we add the attribute-aware loss term to classification loss like softmax. How to combine it with metric learning algorithms like triplet loss used in~\cite{schroff2015facenet} is an interesting research problem.

\section*{Acknowledgments}
This work was supported by National Natural Science Foundation of China (No. 61672481), and Youth Innovation Promotion Association CAS (No. 2018495).

{\small
\bibliographystyle{ieee}
\bibliography{GuidedRecognition}

\begin{thebibliography}{10}\itemsep=-1pt

\bibitem{ABATE20071885}
A.~F. Abate, M.~Nappi, D.~Riccio, and G.~Sabatino.
\newblock 2d and 3d face recognition: A survey.
\newblock {\em Pattern Recognition Letters}, 28(14):1885--1906, 2007.

\bibitem{cui2018rgb}
J.~Cui, H.~Han, S.~Shan, and X.~Chen.
\newblock Rgb-d face recognition: A comparative study of representative fusion
  schemes.
\newblock In {\em Chinese Conference on Biometric Recognition}, pages 358--366.
  Springer, 2018.

\bibitem{deng2018arcface}
J.~Deng, J.~Guo, and S.~Zafeiriou.
\newblock Arcface: Additive angular margin loss for deep face recognition.
\newblock {\em arXiv preprint arXiv:1801.07698}, 2018.

\bibitem{faltemier2008region}
T.~C. Faltemier, K.~W. Bowyer, and P.~J. Flynn.
\newblock A region ensemble for 3-d face recognition.
\newblock {\em IEEE Transactions on Information Forensics and Security},
  3(1):62--73, 2008.

\bibitem{GoswamiVS14}
G.~Goswami, M.~Vatsa, and R.~Singh.
\newblock {RGB-D} face recognition with texture and attribute features.
\newblock {\em {IEEE} Trans. Information Forensics and Security},
  9(10):1629--1640, 2014.

\bibitem{Guo20183DFace}
Y.~Guo, J.~Zhang, J.~Cai, B.~Jiang, and J.~Zheng.
\newblock Cnn-based real-time dense face reconstruction with inverse-rendered
  photo-realistic face images.
\newblock {\em IEEE Transactions on Pattern Analysis and Machine Intelligence},
  2018.

\bibitem{hadsell2006dimensionality}
R.~Hadsell, S.~Chopra, and Y.~LeCun.
\newblock Dimensionality reduction by learning an invariant mapping.
\newblock In {\em Computer vision and pattern recognition, 2006 IEEE computer
  society conference on}, volume~2, pages 1735--1742. IEEE, 2006.

\bibitem{he2016deep}
K.~He, X.~Zhang, S.~Ren, and J.~Sun.
\newblock Deep residual learning for image recognition.
\newblock In {\em Proceedings of the IEEE Conference on Computer Vision and
  Pattern Recognition}, pages 770--778, 2016.

\bibitem{hernandez2015near}
M.~Hernandez, J.~Choi, and G.~Medioni.
\newblock Near laser-scan quality 3-d face reconstruction from a low-quality
  depth stream.
\newblock {\em Image and Vision Computing}, 36:61--69, 2015.

\bibitem{HsuLPW14}
G.~Hsu, Y.~Liu, H.~Peng, and P.~Wu.
\newblock Rgb-d-based face reconstruction and recognition.
\newblock {\em {IEEE} Trans. Information Forensics and Security},
  9(12):2110--2118, 2014.

\bibitem{hu2017attribute}
G.~Hu, Y.~Hua, Y.~Yuan, Z.~Zhang, Z.~Lu, S.~S. Mukherjee, T.~M. Hospedales,
  N.~M. Robertson, and Y.~Yang.
\newblock Attribute-enhanced face recognition with neural tensor fusion
  networks.
\newblock In {\em ICCV}, pages 3764--3773, 2017.

\bibitem{huang2014labeled}
G.~B. Huang and E.~Learned-Miller.
\newblock Labeled faces in the wild: Updates and new reporting procedures.
\newblock {\em Dept. Comput. Sci., Univ. Massachusetts Amherst, Amherst, MA,
  USA, Tech. Rep}, pages 14--003, 2014.

\bibitem{jain2012metric}
P.~Jain, B.~Kulis, J.~V. Davis, and I.~S. Dhillon.
\newblock Metric and kernel learning using a linear transformation.
\newblock {\em Journal of Machine Learning Research}, 13(Mar):519--547, 2012.

\bibitem{jia2014caffe}
Y.~Jia, E.~Shelhamer, J.~Donahue, S.~Karayev, J.~Long, R.~Girshick,
  S.~Guadarrama, and T.~Darrell.
\newblock Caffe: Convolutional architecture for fast feature embedding.
\newblock {\em arXiv preprint arXiv:1408.5093}, 2014.

\bibitem{kemelmacher2016megaface}
I.~Kemelmacher-Shlizerman, S.~M. Seitz, D.~Miller, and E.~Brossard.
\newblock The megaface benchmark: 1 million faces for recognition at scale.
\newblock In {\em Proceedings of the IEEE Conference on Computer Vision and
  Pattern Recognition}, pages 4873--4882, 2016.

\bibitem{kim2017deep}
D.~Kim, M.~Hernandez, J.~Choi, and G.~Medioni.
\newblock Deep 3d face identification.
\newblock {\em arXiv preprint arXiv:1703.10714}, 2017.

\bibitem{krizhevsky2012imagenet}
A.~Krizhevsky, I.~Sutskever, and G.~E. Hinton.
\newblock Imagenet classification with deep convolutional neural networks.
\newblock In {\em Advances in neural information processing systems}, pages
  1097--1105, 2012.

\bibitem{kumar2016learning}
B.~Kumar, G.~Carneiro, I.~Reid, et~al.
\newblock Learning local image descriptors with deep siamese and triplet
  convolutional networks by minimising global loss functions.
\newblock In {\em Proceedings of the IEEE Conference on Computer Vision and
  Pattern Recognition}, pages 5385--5394, 2016.

\bibitem{KumarBBN09}
N.~Kumar, A.~C. Berg, P.~N. Belhumeur, and S.~K. Nayar.
\newblock Attribute and simile classifiers for face verification.
\newblock In {\em {IEEE} 12th International Conference on Computer Vision,
  {ICCV} 2009, Kyoto, Japan, September 27 - October 4, 2009}, pages 365--372,
  2009.

\bibitem{LeeCTL16}
Y.~Lee, J.~Chen, C.~W. Tseng, and S.~Lai.
\newblock Accurate and robust face recognition from {RGB-D} images with a deep
  learning approach.
\newblock In {\em Proceedings of the British Machine Vision Conference 2016,
  {BMVC} 2016, York, UK, September 19-22, 2016}, 2016.

\bibitem{LiXMLK16}
B.~Y.~L. Li, M.~Xue, A.~S. Mian, W.~Liu, and A.~Krishna.
\newblock Robust {RGB-D} face recognition using kinect sensor.
\newblock {\em Neurocomputing}, 214:93--108, 2016.

\bibitem{li2014expression}
H.~Li, D.~Huang, J.-M. Morvan, L.~Chen, and Y.~Wang.
\newblock Expression-robust 3d face recognition via weighted sparse
  representation of multi-scale and multi-component local normal patterns.
\newblock {\em Neurocomputing}, 133:179--193, 2014.

\bibitem{liu2017sphereface}
W.~Liu, Y.~Wen, Z.~Yu, M.~Li, B.~Raj, and L.~Song.
\newblock Sphereface: Deep hypersphere embedding for face recognition.
\newblock In {\em The IEEE Conference on Computer Vision and Pattern
  Recognition (CVPR)}, volume~1, page~1, 2017.

\bibitem{liu2016large}
W.~Liu, Y.~Wen, Z.~Yu, and M.~Yang.
\newblock Large-margin softmax loss for convolutional neural networks.
\newblock In {\em Proceedings of The 33rd International Conference on Machine
  Learning}, pages 507--516, 2016.

\bibitem{liu2015deep}
Z.~Liu, P.~Luo, X.~Wang, and X.~Tang.
\newblock Deep learning face attributes in the wild.
\newblock In {\em Proceedings of the IEEE International Conference on Computer
  Vision}, pages 3730--3738, 2015.

\bibitem{Maaten08vd}
L.~v.~d. Maaten and G.~Hinton.
\newblock Visualizing data using t-sne.
\newblock {\em Journal of machine learning research}, 9(Nov):2579--2605, 2008.

\bibitem{ParkhiVZ15}
O.~M. Parkhi, A.~Vedaldi, and A.~Zisserman.
\newblock Deep face recognition.
\newblock In {\em Proceedings of the British Machine Vision Conference 2015,
  {BMVC} 2015, Swansea, UK, September 7-10, 2015}, pages 41.1--41.12, 2015.

\bibitem{phillips2005overview}
P.~J. Phillips, P.~J. Flynn, T.~Scruggs, K.~W. Bowyer, J.~Chang, K.~Hoffman,
  J.~Marques, J.~Min, and W.~Worek.
\newblock Overview of the face recognition grand challenge.
\newblock In {\em Computer vision and pattern recognition, 2005. CVPR 2005.
  IEEE computer society conference on}, volume~1, pages 947--954. IEEE, 2005.

\bibitem{ranjan2017l2}
R.~Ranjan, C.~D. Castillo, and R.~Chellappa.
\newblock L2-constrained softmax loss for discriminative face verification.
\newblock {\em arXiv preprint arXiv:1703.09507}, 2017.

\bibitem{RanjanSCC17}
R.~Ranjan, S.~Sankaranarayanan, C.~D. Castillo, and R.~Chellappa.
\newblock An all-in-one convolutional neural network for face analysis.
\newblock In {\em 12th {IEEE} International Conference on Automatic Face {\&}
  Gesture Recognition}, pages 17--24, 2017.

\bibitem{richardson2016learning}
E.~Richardson, M.~Sela, R.~Or{-}El, and R.~Kimmel.
\newblock Learning detailed face reconstruction from a single image.
\newblock In {\em 2017 {IEEE} Conference on Computer Vision and Pattern
  Recognition, {CVPR} 2017, Honolulu, HI, USA, July 21-26, 2017}, pages
  5553--5562, 2017.

\bibitem{rudd2016moon}
E.~M. Rudd, M.~G{\"u}nther, and T.~E. Boult.
\newblock Moon: A mixed objective optimization network for the recognition of
  facial attributes.
\newblock In {\em European Conference on Computer Vision}, pages 19--35.
  Springer, 2016.

\bibitem{SamangoueiC16}
P.~Samangouei and R.~Chellappa.
\newblock Convolutional neural networks for attribute-based active
  authentication on mobile devices.
\newblock In {\em 8th {IEEE} International Conference on Biometrics Theory,
  Applications and Systems, {BTAS} 2016, Niagara Falls, NY, USA, September 6-9,
  2016}, pages 1--8, 2016.

\bibitem{savran2008bosphorus}
A.~Savran, N.~Aly{\"u}z, H.~Dibeklio{\u{g}}lu, O.~{\c{C}}eliktutan,
  B.~G{\"o}kberk, B.~Sankur, and L.~Akarun.
\newblock Bosphorus database for 3d face analysis.
\newblock In {\em European Workshop on Biometrics and Identity Management},
  pages 47--56. Springer, 2008.

\bibitem{schroff2015facenet}
F.~Schroff, D.~Kalenichenko, and J.~Philbin.
\newblock Facenet: A unified embedding for face recognition and clustering.
\newblock In {\em Proceedings of the IEEE Conference on Computer Vision and
  Pattern Recognition}, pages 815--823, 2015.

\bibitem{sharif2014cnn}
A.~Sharif~Razavian, H.~Azizpour, J.~Sullivan, and S.~Carlsson.
\newblock Cnn features off-the-shelf: an astounding baseline for recognition.
\newblock In {\em Proceedings of the IEEE Conference on Computer Vision and
  Pattern Recognition Workshops}, pages 806--813, 2014.

\bibitem{sohn2016improved}
K.~Sohn.
\newblock Improved deep metric learning with multi-class n-pair loss objective.
\newblock In {\em Advances in Neural Information Processing Systems}, pages
  1849--1857, 2016.

\bibitem{sun2013deep}
Y.~Sun, X.~Wang, and X.~Tang.
\newblock Deep convolutional network cascade for facial point detection.
\newblock In {\em Proceedings of the IEEE conference on computer vision and
  pattern recognition}, pages 3476--3483, 2013.

\bibitem{szegedy2015going}
C.~Szegedy, W.~Liu, Y.~Jia, P.~Sermanet, S.~Reed, D.~Anguelov, D.~Erhan,
  V.~Vanhoucke, and A.~Rabinovich.
\newblock Going deeper with convolutions.
\newblock In {\em Proceedings of the IEEE Conference on Computer Vision and
  Pattern Recognition}, pages 1--9, 2015.

\bibitem{taigman2014deepface}
Y.~Taigman, M.~Yang, M.~Ranzato, and L.~Wolf.
\newblock Deepface: Closing the gap to human-level performance in face
  verification.
\newblock In {\em Proceedings of the IEEE Conference on Computer Vision and
  Pattern Recognition}, pages 1701--1708, 2014.

\bibitem{vijayan2011twins}
V.~Vijayan, K.~W. Bowyer, P.~J. Flynn, D.~Huang, L.~Chen, M.~Hansen,
  O.~Ocegueda, S.~K. Shah, and I.~A. Kakadiaris.
\newblock Twins 3d face recognition challenge.
\newblock In {\em Biometrics (IJCB), 2011 International Joint Conference on},
  pages 1--7. IEEE, 2011.

\bibitem{wang2018additive}
F.~Wang, J.~Cheng, W.~Liu, and H.~Liu.
\newblock Additive margin softmax for face verification.
\newblock {\em IEEE Signal Processing Letters}, 25(7):926--930, 2018.

\bibitem{wang2017normface}
F.~Wang, X.~Xiang, J.~Cheng, and A.~L. Yuille.
\newblock Normface: L\({}_{\mbox{2}}\) hypersphere embedding for face
  verification.
\newblock In {\em Proceedings of the 2017 {ACM} on Multimedia Conference, {MM}
  2017, Mountain View, CA, USA, October 23-27, 2017}, pages 1041--1049, 2017.

\bibitem{Wang2011kernel}
J.~Wang, H.~T. Do, A.~Woznica, and A.~Kalousis.
\newblock Metric learning with multiple kernels.
\newblock In J.~Shawe-Taylor, R.~S. Zemel, P.~L. Bartlett, F.~Pereira, and
  K.~Q. Weinberger, editors, {\em Advances in Neural Information Processing
  Systems 24}, pages 1170--1178. Curran Associates, Inc., 2011.

\bibitem{weinberger2009distance}
K.~Q. Weinberger and L.~K. Saul.
\newblock Distance metric learning for large margin nearest neighbor
  classification.
\newblock {\em Journal of Machine Learning Research}, 10(Feb):207--244, 2009.

\bibitem{wen2016discriminative}
Y.~Wen, K.~Zhang, Z.~Li, and Y.~Qiao.
\newblock A discriminative feature learning approach for deep face recognition.
\newblock In {\em European Conference on Computer Vision}, pages 499--515.
  Springer, 2016.

\bibitem{wolf2011face}
L.~Wolf, T.~Hassner, and I.~Maoz.
\newblock Face recognition in unconstrained videos with matched background
  similarity.
\newblock In {\em Computer Vision and Pattern Recognition (CVPR), 2011 IEEE
  Conference on}, pages 529--534. IEEE, 2011.

\bibitem{xiang2015light}
Z.~S. Xiang~Wu, Ran~He.
\newblock A lightened cnn for deep face representation.
\newblock {\em arXiv preprint arXiv:1511.02683v1}, 2015.

\bibitem{xing2002distance}
E.~P. Xing, A.~Y. Ng, M.~I. Jordan, and S.~Russell.
\newblock Distance metric learning, with application to clustering with
  side-information.
\newblock In {\em Proceedings of the 15th International Conference on Neural
  Information Processing Systems}, NIPS'02, pages 521--528, 2002.

\bibitem{yi2014learning}
D.~Yi, Z.~Lei, S.~Liao, and S.~Z. Li.
\newblock Learning face representation from scratch.
\newblock {\em arXiv preprint arXiv:1411.7923}, 2014.

\bibitem{HCSC18}
H.~Zhang, H.~Han, J.~Cui, S.~Shan, and X.~Chen.
\newblock {RGB-D} face recognition via deep complementary and common feature
  learning.
\newblock In {\em 13th {IEEE} International Conference on Automatic Face {\&}
  Gesture Recognition, {FG} 2018, Xi'an, China, May 15-19, 2018}, pages 8--15,
  2018.

\bibitem{zhang2016joint}
K.~Zhang, Z.~Zhang, Z.~Li, and Y.~Qiao.
\newblock Joint face detection and alignment using multitask cascaded
  convolutional networks.
\newblock {\em IEEE Signal Processing Letters}, 23(10):1499--1503, 2016.

\bibitem{zulqarnain2018learning}
S.~Zulqarnain~Gilani and A.~Mian.
\newblock Learning from millions of 3d scans for large-scale 3d face
  recognition.
\newblock In {\em Proceedings of the IEEE Conference on Computer Vision and
  Pattern Recognition}, pages 1896--1905, 2018.

\end{thebibliography}
}

\begin{IEEEbiography}[{\includegraphics[width=1in]{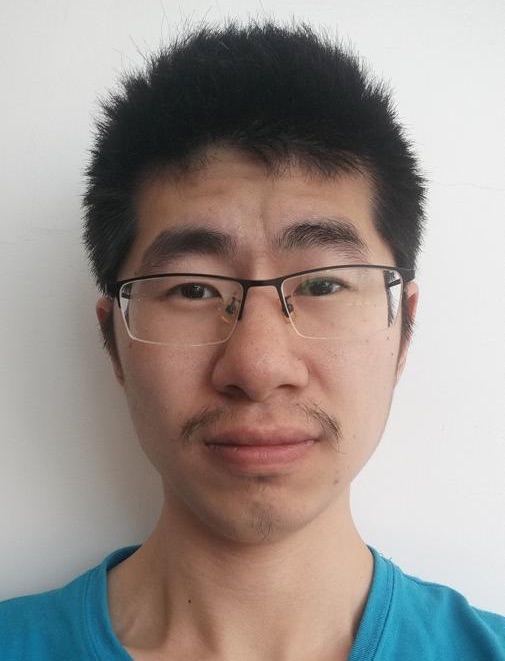}}]{Luo Jiang}
is currently working towards the PhD degree at the University of Science and Technology of China.
He obtained his bachelor degree in 2013 from the Huazhong University of Science and Technology, China. His research interests include computer graphics, image processing and deep learning.
\end{IEEEbiography}

\vspace{-5mm}
\begin{IEEEbiography}[{\includegraphics[width=1in]{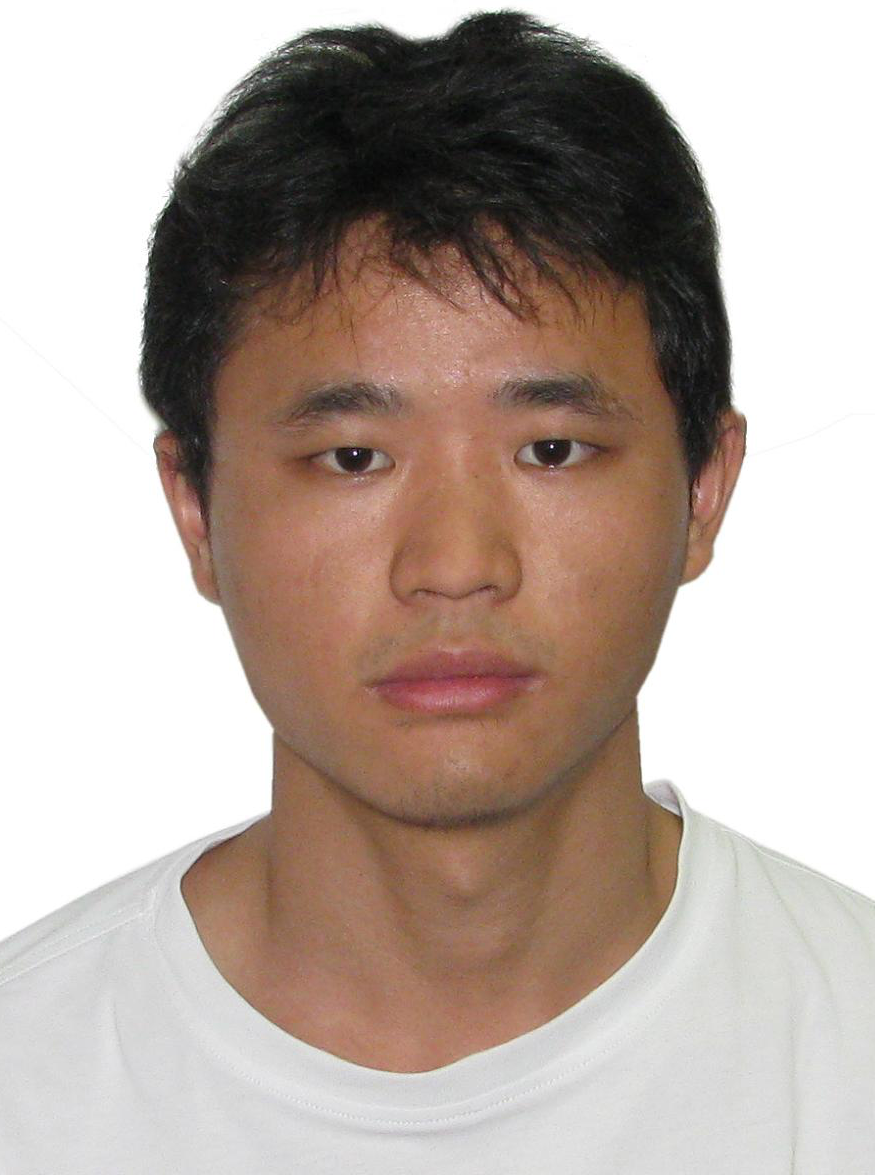}}]{Juyong Zhang}
is an associate professor in the School of Mathematical Sciences at University of Science and Technology of China. He received the BS degree from the University of Science and Technology of China in 2006, and the PhD degree from Nanyang Technological University, Singapore. His research interests include computer graphics, computer vision, and numerical optimization. He is an associate editor of The Visual Computer.
\end{IEEEbiography}

\vspace{-5mm}
\begin{IEEEbiography}[{\includegraphics[width=1in]{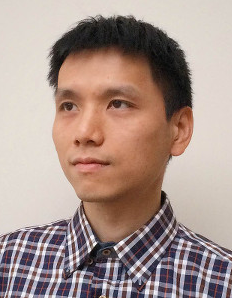}}]{Bailin Deng}
is a lecturer in the School of Computer Science and Informatics at Cardiff University. He received the BEng degree in computer software (2005) and the MSc degree in computer science (2008) from Tsinghua University (China), and the PhD degree in technical mathematics from Vienna University of Technology (Austria). His research interests include geometry processing, numerical optimization, computational design, and digital fabrication. He is a member of the IEEE.
\end{IEEEbiography}

\end{document}